\documentclass{article}

\usepackage{graphicx}
\usepackage{epsfig}

\usepackage{moreverb,url}
\usepackage[colorlinks,bookmarksopen,bookmarksnumbered,citecolor=red,urlcolor=blue]{hyperref}
\newcommand\BibTeX{{\rmfamily B\kern-.05em \textsc{i\kern-.025em b}\kern-.08em
T\kern-.1667em\lower.7ex\hbox{E}\kern-.125emX}}

\usepackage[usenames,dvipsnames,svgnames,table]{xcolor}
\usepackage{wrapfig}
\usepackage{times}
\usepackage{multicol}
\usepackage{multirow}
\usepackage{amsmath,amsthm,amssymb}
\usepackage{pgf,tikz}
\usepackage{mathrsfs}
\usepackage[linesnumbered,ruled,vlined]{algorithm2e}
\usepackage{overpic}
\usepackage{xcolor}
\usepackage{xargs} 
\usepackage{verbatim} 
\usepackage[textsize=footnotesize]{todonotes}
\usepackage{enumerate}
\usepackage{placeins}
\usepackage{marvosym}
\usepackage{tabularx}
\usepackage{xspace}
\usepackage{MnSymbol} 
\usepackage[normalem]{ulem}

\newtheorem{lemma}{Lemma}[section]
\newtheorem{proposition}{Proposition}[section]

\newtheorem{corollary}{Corollary}[section]
\newtheorem{theorem}{Theorem}[section]
\theoremstyle{definition}

\newtheorem{remark}{Remark}

\newif\ifarxiv
\arxivfalse

\newif\iftwocolumn
\twocolumntrue

\usepackage{mathtools,xparse}

\SetKwProg{Fn}{Function}{}{}
\SetKwComment{Comment}{$\triangleright$\ }{}

\newcommand{\namo}{{\texttt{NAMO}}\xspace}

\newcommand{\plor}{{\texttt{LOR}}\xspace}
\newcommand{\ppor}{{\texttt{POR}}\xspace}
\newcommand{\pltr}{{\texttt{LTR}}\xspace}
\newcommand{\pptr}{{\texttt{PTR}}\xspace}

\begin{document}
%\runninghead{Yu}
\date{}
\title{Rearrangement on Lattices with Pick-n-Swaps: Optimality Structures and Efficient Algorithms}
\author{Jingjin Yu}%\affilnum{1}}
\maketitle
%\affiliation{\affilnum{Department of Computer Science, Rutgers University}}

%\corrauth{Jingjin Yu, 
%110 Frelinghuysen Road,
%Piscataway, NJ, 08854-8019 USA.
%}

%\email{jingjin.yu@cs.rutgers.edu}
%\renewcommand\UrlFont{\color{blue}\fontfamily{lmtt}\selectfont}
\begin{abstract}
We study a class of rearrangement problems under a novel pick-n-swap prehensile manipulation model, in which a robotic manipulator, capable of carrying an item and making item swaps, is tasked to  sort items stored in lattices of variable dimensions in a time-optimal manner. 
We systematically analyze the intrinsic optimality structure, which is fairly rich and intriguing, under different levels of item distinguishability (fully labeled, where each item has a unique label, or partially labeled, where multiple items may be of the same type) and different lattice dimensions.
Focusing on the most practical setting of one and two dimensions, we develop low polynomial time cycle-following-based algorithms that optimally perform rearrangements on 1D lattices under both fully- and partially-labeled settings.
On the other hand, we show that rearrangement on 2D and higher-dimensional lattices become computationally intractable to optimally solve. 
Despite their NP-hardness, we prove that efficient cycle-following-based algorithms remain optimal in the asymptotic sense for 2D fully- and partially-labeled settings, in expectation, using the interesting fact that random permutations induce only a small number of cycles. 
We further improve these algorithms to provide $1.x$-optimality when the number of items is small. 
Simulation studies corroborate the effectiveness of our algorithms.
The implementation of the algorithms from the paper can be found at \href{https://github.com/arc-l/lattice-rearrangement/}{\nolinkurl{github.com/arc-l/lattice-rearrangement}}.
\end{abstract}

%\keywords{Rearrangement, Manipulation Task Planning}

\section{Introduction}\label{sec:intro}
Effective object manipulation \cite{mason2018toward}, a difficult task and motion planning challenge for machines to master, is essential in fulfilling the true potential of autonomous robots in factories and warehouses, and at home. 
In tackling the challenge, in the past few decades, while some research has emphasized integrated solutions with promising results\cite{kaelbling2011hierarchical,levine2016end,mahler2017dex,zeng2022robotic}, significant efforts have also been devoted to examining key components including rearrangement planning \cite{ben1998practical,stilman2005navigation,treleaven2013asymptotically,havur2014geometric,haustein2015kinodynamic,krontiris2015dealing,king2016rearrangement,shome2021fast,han2018complexity,huang2019large,lee2019efficient,pan2020decision} and manipulation \cite{goldberg1993orienting,lynch1999dynamic,dogar2011framework,boularias2015learning,chavan2015prehensile}, among others, as a thorough understanding of these components is indispensable toward the end goal of enabling truly intelligent object manipulation. 
\begin{figure}[ht]
\begin{center}
\vspace{1mm}
\begin{overpic}[width={\iftwocolumn 0.99\columnwidth \else 5in \fi},tics=5]
{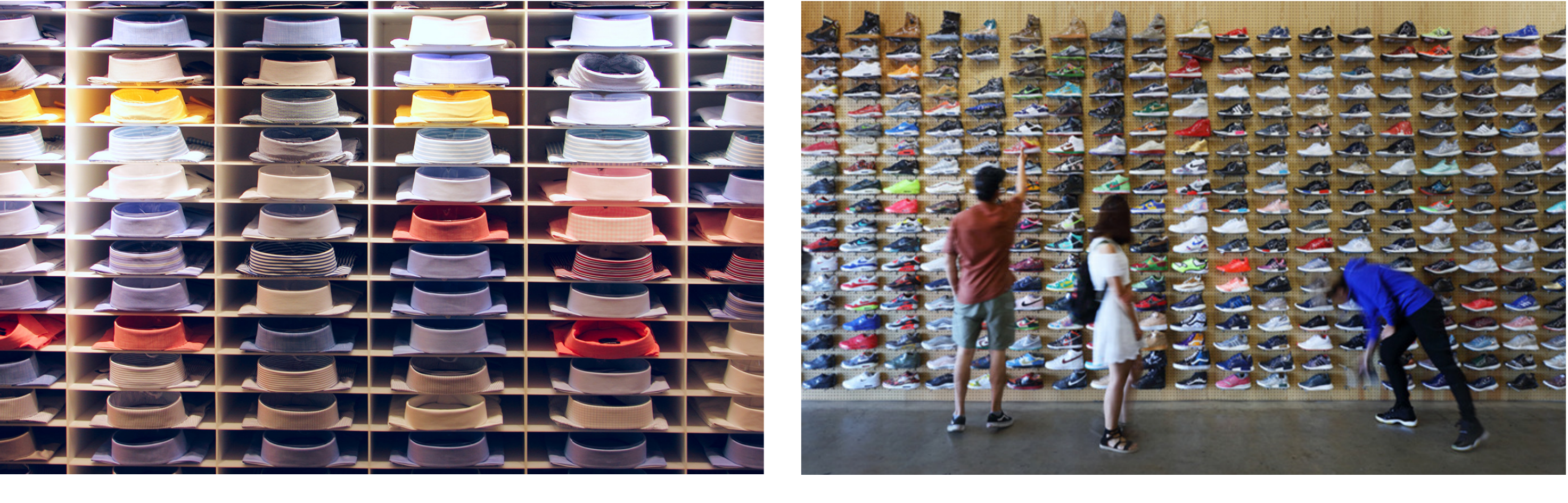}
\end{overpic}
\begin{overpic}[width={\iftwocolumn 1\columnwidth \else 5in \fi},tics=5]
{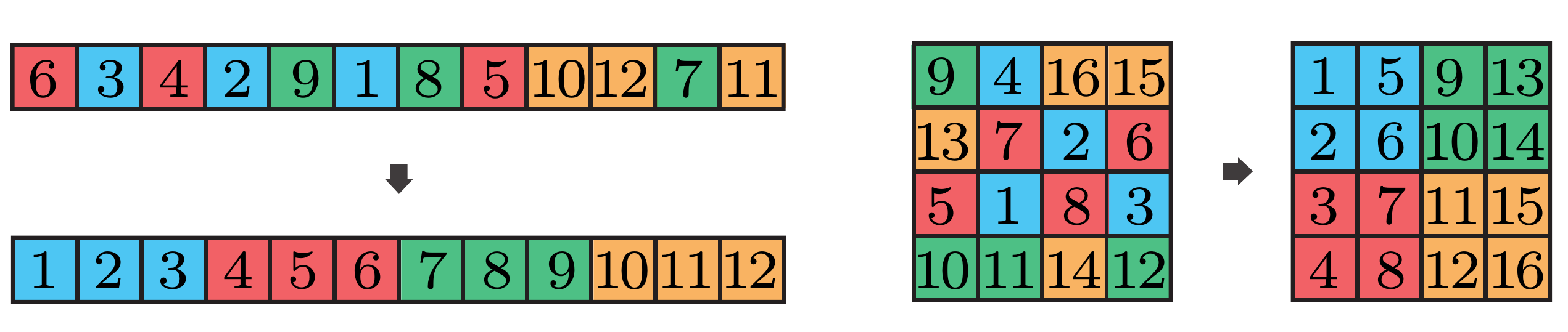}
\end{overpic}
\end{center}
\caption{\label{fig:exs} [top row] Real-world examples of items stored in lattice/grids that need rearrangement from time to time. [top left] Shirts. [top right] Shoe display. [bottom row] Examples of the rearrangement problem formulations studied in this work. [bottom left] A set of $12$ items in a row that must be rearranged either according to the labels or according to the types (colors). [bottom right]  A set of $16$ items in a two-dimensional lattice that must be rearranged either according to the labels or according to the types (colors).}
\end{figure}

For the same reason, in this work, we perform a systematic structural and algorithmic study on a class of prehensile rearrangement problems where items are stored in individual cells of a lattice of dimension $d = 1, 2, \ldots$ (see Fig.~\ref{fig:exs}, bottom row, for example start and goal configurations), under a novel \emph{pick-n-swap} model. 
The stored items may be fully labeled or partially labeled. 
In a fully-labeled setting, each item has a unique label and must go to a specific lattice cell. In a partially-labeled setting, multiple items may have the same label or type and are thus interchangeable.
A robotic manipulator, capable of picking up items, carrying them around (a single item at a time), and executing item swaps, is tasked to rearrange items to reach a desired goal configuration in a time-optimal manner. Efficient solutions for rearrangement problems on lattices to restore the items to a desired order find many practical applications, including the rearrangement of products at stores and showrooms (see Fig.~\ref{fig:exs}, top row), the sorting of books on bookshelves, and inventory management in autonomous vertical warehouses, to list a few.
To accomplish the task, the robot must carefully plan a sequence with which the items are picked and subsequently placed, to minimize the number of pick-n-swaps and the associated end-effector travel. 

Our study of the rearrangement problems under the pick-n-swap prehensile manipulation model is motivated by the recent emergence of dual-gripper end-effectors. Such setups have clear potential in rendering rearrangement tasks more efficiently solvable, given the capacity of carrying more than one item at a time. A dual-gripper end-effector provides functionality that falls between a single-arm-single-gripper setup and a dual-arm setup, leaning closer to a dual-arm setup. While not as flexible as a dual-arm setup, a dual-gripper setup has at least two distinct advantages: it is much more affordable and it does not require a sophisticated motion planning routine to coordinate arm movements. Whereas this research limits the workspace to be lattice-like, the structures analyzed and algorithms developed here can be extended to enable dual-gripper end-effectors to be more efficient in a fully continuous workspace as well.

Due to the intricate interaction between minimizing the number of pick-n-swaps and minimizing the end-effector travel, time-optimal rearrangement on lattices demonstrates a rich and complex structure.
In exploring the structure, our study of the lattice rearrangement problem under the pick-n-swap model brings forth the following contributions:
\begin{itemize}
	\item For the fully-labeled setting, we show that a random rearrangement instance induces the forming of $O(\log m)$ \emph{cycles} over $m$ items. Each cycle decides an optimal pick-n-swap sequence within the cycle in a deterministic manner. 
	For the partially-labeled setting, similar but more complex cycle structures are present. 
	The combinatorial cycle structure applies generally to rearrangement problems, which go beyond lattice-based settings. 
	\item Building on the intriguing and easily computable cycle structure, 	for 1D lattices, for the fully-labeled setting, efficient algorithms are developed that decide the optimal pick-n-swap sequence which simultaneously minimizes the number of pick-n-swap operations and the end-effector travel. For the partially-labeled setting, our algorithms compute optimal pick-n-swap sequences that minimize the end-effector	travel after minimizing the number of pick-n-swaps and allow trade-offs between the two component objectives. 
	\item For 2D and high-dimensional lattices, we show that minimizing the end-effector travel becomes NP-hard for both fully-labeled and partially-labeled settings. This demonstrates a dichotomy in computational	complexity between 1D and higher dimensions for optimally solving the lattice rearrangement problem.
	\item For 2D and high-dimensional lattices, despite the hardness, 	we show that, because the number of cycles is small, efficient algorithms can be developed that compute optimal solutions in the asymptotic sense. 
	That is, as the number of items goes to infinity, the solution converges to the optimal one, in expectation. 
	For the 2D setting, we further develop additional principled heuristics such that even when the number of items is small, the solution is near-optimal, i.e., $1.x$-optimal. 
\end{itemize}
This paper is based on the conference publication \cite{Yu21RSS}. In comparison to the conference version, this manuscript provides: (1) a more thorough and clearer presentation, including complete,  revised proofs for all theorems, (2) significantly stronger guarantees on our algorithms for the labeled settings as well as many unlabeled settings, showing that they often ensure global optimality (Remark~\ref{r:global}, Corollary~\ref{c:plor-global}, Remark~\ref{r:por}, Corollary~\ref{c:pltr-global}, and Remark~\ref{r:pptr}), and (3) expanded discussion of alternative problem formulations and their implications, and interesting open problems. 

\textbf{Related work}. Multi-object rearrangement is computationally challenging. As a variation of multi-robot motion planning problems,  rearrangement inherits the PSPACE-hard complexity \cite{hopcroft1984complexity}. When geometric constraints must be 
considered, the relatively simple Navigation Among Movable Obstacle (\namo) problem is shown to be NP-hard \cite{wilfong1991motion}. If consideration of plan quality is further required, as is often the case in practice, optimally resolving dependency \cite{van2009centralized} or planning an optimal object pick-n-place sequence is both NP-hard \cite{han2018complexity}. Rearrangement problems addressed in this paper have a somewhat similar complexity structure. 

Despite the high computational complexity, due to its high utility, multi-object rearrangement has been extensively studied, with much research working with non-prehensile (e.g., pushing) manipulations, sometimes assisted with prehensile (grasping) actions. 
A complete sensing-planning-control framework is proposed in \cite{chang2012interactive} for the singulation of objects in clutter, which uses both perturbations pushes and grasping actions.
In \cite{laskey2016robot}, hierarchical supervised learning from demonstration is applied to the singulation task. 
Based on over-segmented RGB-D images, in \cite{eitel2020learning}, a push proposal network is constructed for push-only singulation. 
Results that bridge singulation and clutter removal include \cite{zeng2018learning,HuaHanBouYu21ICRA}, where learning-based methods are trained to dictate when to push or grasp. 
To deal with the combinatorial explosions inherent in rearrangement, 
a randomized kinodynamic planner is employed in \cite{haustein2015kinodynamic} for rearranging objects on a tabletop, allowing the effective exploration of configurations. 
\cite{king2015nonprehensile} further integrates a physics-based model to enable the use of the entire robotic arm for complex rearrangement manipulations. 
A physics-based approach is also used in \cite{moll2017randomized} for handling grasping in clutter. 
For a similar task, \cite{bejjani2018planning} uses a receding horizon planner with a learned value function that interleaves planning and plan execution. 

\cite{huang2019large} has developed an Iterated Local Search (ILS) method for accomplishing  multiple tabletop rearrangement tasks including singulation, separation, formation, and sorting of many identically shaped cubes. 
In \cite{song2019multi}, Monte Carlo Tree Search is combined with deep learning for separating many objects into coherent clusters 
within a bounded workspace. In contrast to \cite{huang2019large}, non-convex objects are supported. 
More recently, \cite{pan2020decision} proposed a bi-level planner that employs both pushing and overhand grasping that is capable of sorting up to $200$ objects. 

On work that uses mainly prehensile actions, the earliest is perhaps the study of \namo problems \cite{stilman2005navigation,stilman2007manipulation}, which applies backtracking search to effectively deal with monotone and linear \namo instances, among others.
Exploring the dependency graph structure, difficult non-monotone tabletop rearrangement instances are solved using monotone solvers as subroutines \cite{krontiris2015dealing,krontiris2016efficiently}.
\cite{han2018complexity} shows that tabletop rearrangement embeds a Feedback Vertex Set (FVS) problem \cite{karp1972reducibility} and the Traveling Salesperson Problem (TSP) \cite{papadimitriou1977euclidean}, both of which are NP-hard, rendering optimally solving these problems intractable. Nevertheless, integer programming models are provided that can quickly compute high-quality rearrangement solutions for practical-sized problem instances. 
In exploring object dependency structures, studies like \cite{krontiris2015dealing,han2018complexity} put more emphasis on the combinatorial aspects of object rearrangement. 
In \cite{lee2019efficient}, a polynomial-time, complete planner for reasoning rearrangement for object retrieval in a constrained, shelf-like setting is proposed. In a subsequent study \cite{nam2019planning}, the number of objects to be relocated for retrieval is minimized while considering sensor occlusion. 

\textbf{Organization}. The rest of the manuscript is organized as follows. In Sec.~\ref{sec:problem}, we formally define the lattice rearrangement problems studied in this paper. In Sec.~\ref{sec:lor} and Sec.~\ref{sec:por}, we provide structural analysis and describe algorithms for the fully- and partially-labeled settings in 1D, respectively. 2D and higher dimensions are examined in Sec.~\ref{sec:ltrptr}. We then empirically characterize the behavior of our algorithms through simulation studies in 
Sec.~\ref{sec:eval}, and discuss and conclude in Sec.~\ref{sec:conclusion}.

\section{Rearrangement on Lattices: Problem Formulation}\label{sec:problem}
We examine an object rearrangement problem where items are stored in a $d$-dimensional lattice or grid, $d = 1, 2, \ldots$, with dimension $i$ 
having a length or capacity of $m_i$. 
Items are assumed to be stored at full capacity, i.e., an $m_1 \times \ldots \times m_d$ lattice stores $\Pi_{i = 1}^d m_i$ items (see, e.g., Fig.~\ref{fig:exs}). 
The items may be \emph{fully labeled} or \emph{partially labeled}. In a the fully-labeled setting, each item has a unique label from the set $\{1, \ldots, \Pi_{i = 1}^dm_i\}$ and must be relocated to a specific location or coordinate on the lattice. In a partially-labeled setting, there are $k> 1$ types of items where items within each type are considered interchangeable; items of the same type are to be grouped, via rearrangement, into contiguous clusters.

It is assumed that a robotic manipulator is capable of picking up an item, temporarily holding it, moving it to a different location, and swapping the held item with the item at the new location. That is, in a single \emph{pick-n-swap} operation where a robot end-effector is located above a fixed lattice coordinate, the robot may execute one of the following: 
\begin{itemize}
\item pick up an item and hold it with the robot's end-effector (only if no item is already held by the robot), 
\item swap the item held by the robot's end-effector with an item inside the lattice at the given coordinate, or
\item place the item held by the robot's end-effector at the lattice coordinate if no item is already at the coordinate.
\end{itemize}
In this paper, we use \emph{lattice coordinate} and \emph{cell} interchangeably. Such a pick-n-swap model can be readily realized, for example, using two adjacent suction cups, two parallel grippers mounted side-by-side, and so on. The pick-n-swap primitive also models, to a lesser extent, a dual-arm robot or for that matter, how humans perform such rearrangement tasks. 
We have also examined alternative pick-n-place models, which are discussed in Sec.~\ref{sec:conclusion}.

The pick-n-swap model leads to a natural partition of the robot's operations into \emph{pick-n-swap operations} and \emph{end-effector travel operations}. A \emph{rearrangement plan} can then be represented as a sequence of lattice coordinates, $P=\{p_0, p_1, \ldots, p_N\}$, where the robot end-effector starts from the rest position $p_0$ and sequentially executes pick-n-swap operations at $p_1, p_2$, and so on. For quantifying the quality of a rearrangement plan $P$, it is assumed that each pick-n-swap incurs a (time) cost of $c_p$ and the (time) cost of traveling a unit distance (the distance between two adjacent cells) by the end-effector is $c_t$. The total cost of completing a rearrangement plan is then 
\begin{align}\label{eq:cost}
J_T(P) = Nc_p + \sum_{i = 0}^{N} dist(p_i, p_{i+1})c_t,
\end{align}
where $dist(p_i, p_{i+1})$ is the effective distance traveled by the end-effector between $p_i$ and $p_{i+1}$; $p_{N+1} = p_0$. The distance metric may be $L_1$, Euclidean ($L_2$), and so on, depending on the end-effector's motion mechanism. For example, if a human is to arrange shoes for the setup shown in Fig.~\ref{fig:exs} [top right],
then horizontal travel is likely the main source of travel distance cost. 
This study works with Euclidean distances, i.e., $dist(p_i, p_{i+1}) 
= \|p_i - p_{i+1}\|_2$. 

\begin{remark}It is straightforward to observe that plans computed for the pick-n-swap model are reversible; the associated optimality guarantees, if any, will also carry over. This implies that the algorithms developed in this work can also be applied to (near-)optimally randomize the locations of items stored in lattices. 
\end{remark}

Because a pick-n-swap operation requires precise robot arm placement and grasp planning/execution involving contact between the end-effector and objects, similar to \cite{han2018complexity}, it is assumed that the total pick-n-swap cost dominates the total end-effector travel cost. That is, on the right side of Eq.~\eqref{eq:cost}, the first term will be considered first, yielding a sequential optimization problem in which minimizing the number of pick-n-swaps, $N$, takes priority. 

\begin{remark}\label{r:asymptotic}
In this paper, we use ``optimal in the asymptotic sense'' to mean that the optimality ratio will approach $1$ as the number of items goes to infinity, as is commonly used in \emph{asymptotic analysis} in mathematics.
This is not to be taken as the same as \emph{asymptotically optimal algorithms} used in computer science, which generally means that an algorithm computes $O(1)$-optimal solutions in the worst case.  
\end{remark}

\begin{remark}\label{r:global} In general, the sequential optimization of two objectives that are not orthogonal to each other does not yield globally optimal solutions for the sum of the two objectives. In our case, however, the sequential optimization procedure leads to global optimality for the fully-labeled settings and some partially-labeled settings. For all partially-labeled settings, our algorithms can also be adjusted to globally balance between more pick-n-swaps and more end-effector travel based on the ratio $c_p:c_t$. 
\end{remark}

As practical robotic rearrangement operations are generally limited to one and two dimensions, our study also centers on the cases of $d=1, 2$, with some discussions of higher dimensions. In the $d = 1$ case, let $m_1 = m$ be the capacity of the lattice, viewed as a single \emph{row}. It is assumed that $p_0$ is at the leftmost cell. In the labeled setting, the lattice is equivalent to a row with its cells labeled $1, \ldots, m$ from left to right; the rearrangement problem is then to relocate the item with label $i$ to the $i$-th cell. In the partially-labeled setting, there are $k$ types of items, $k < m$, possibly to be arranged into contiguous clusters (see, e.g., Fig.~\ref{fig:exs} [lower left]), though we do not make assumptions about the goal configurations. 

In the $d = 2$ case, we have an $m_1 (\mathrm{row}) \times m_2 (\mathrm{column})$ lattice; $p_0$ is at the top left cell of the lattice. In the labeled setting, it is assumed without loss of generality that lattice cells are labeled following a column-major order: cells in column $i, 1\le i \le m_2$, are labeled $(i - 1)*m_1 + 1, \ldots, im_1$, from top to bottom, respectively. In the goal configuration, the item labeled $j$ must be located at cell $j$. In the partially-labeled setting, besides considering arbitrary goal configurations, two natural goal configuration patterns are analyzed in more detail, with one having the goals of the same type aggregated (e.g., Fig.~\ref{fig:exs} [bottom right]) and the other having each type occupying a single lattice column.

For convenience, we denote the one-dimensional labeled and partially-labeled problems, as stated above, as the \textbf{l}abeled \textbf{o}ne-dimensional \textbf{r}earrangement (\plor) problem and the \textbf{p}artially-labeled \textbf{o}ne-dimensional \textbf{r}earrangement (\ppor) problem, respectively. The \textbf{t}wo-dimensional problems corresponding to \plor and \ppor are named as \pltr and \pptr, respectively. In this study, we analyze the structural properties of \plor/\ppor/\pltr/\pptr as induced by minimizing the objective specified in Eq.~\eqref{eq:cost}. Based on the findings, we design efficient algorithms for computing (near)-optimal plans for these rearrangement problems.

\section{Fully-Labeled Rearrangement in 1D}\label{sec:lor}
\def\cyclesweep{\textsc{SweepCyclesLOR}\xspace}
\def\cstd{\textsc{SweepCyclesLTR}\xspace}
\def\cswtd{\textsc{SwitchCyclesLTR}\xspace}
\def\cycleGS{\textsc{OptPlanLOR}\xspace}
\def\swap{\textsc{Swap}\xspace}

For combinatorial optimization problems involving the reconfiguration of many bodies, the labeled settings are often more challenging (e.g., \cite{solovey2014k,han2018complexity}). In arranging items stored in lattices, in contrast, the labeled case is structurally simpler. This is due to the ``linear'' dependencies among the items to be relocated, as will become clear shortly. This allows the computation of (exact) optimal solutions for \plor. 

Recall that \plor requires rearranging a row of $m$ items. Therefore, \plor can be viewed as going from one random permutation of $m$ labeled items to the canonical order $[m] := 1, \ldots, m$. An \plor instance is therefore fully specified by a permutation $\pi$ of $[m]$, where $\pi_i$, $1 \le i \le m$ is the label of the item that occupies cell $i$ in the initial configuration. We start with a simple \emph{greedy} \emph{cycle-following} algorithm, \cyclesweep, that solves \plor near-optimally, in expectation, assuming that the initial \plor instance is randomly generated. \cyclesweep minimizes the end-effector travel distance to near-optimality after minimizing the number of pick-n-swap operations to the smallest possible.
Then, we describe a more involved algorithm, \cycleGS, that computes a rearrangement plan that also minimizes the end-effector travel. 
Unlike \cyclesweep, \cycleGS guarantees solution optimality for each individual instance, i.e., it does not depend on the \plor instance being random. 
After presenting \cyclesweep and \cycleGS, we further show that although they perform sequential optimization of Eq.~\eqref{eq:cost}, their optimality guarantees are in fact global.

\subsection{Cycle Following with Left to Right Sweeping}
Consider an \plor instance with $9$ items and the initial configuration $\pi = (3,2,4,1,7,6,9,5,8)$. The instance can be solved by starting with the leftmost item that needs rearrangement, in this case $3$, and moving it from cell $1$ to cell $3$, which replaces item $4$ that in turn replaces 
item $1$. This yields a \emph{cycle} $3-4-1$ (see Fig.~\ref{fig:cf-ex}), often written as $(341)$.
After following this cycle, the end-effector returns to cell $1$. The end-effector then works with the next leftmost item that is not at goal, $7$, inducing another cycle $7985$. Altogether, there are two cycles, $(341)$ and $(7985)$, containing $7$ items in total (here, we deviate slightly from how cycles are normally counted in permutations, where a single item in the correct cell would be counted as a cycle containing a single element as well, which we ignore by default).

\begin{figure}[h]
\begin{center}
\begin{overpic}[width={\iftwocolumn 1\columnwidth \else 5in \fi},tics=5]
{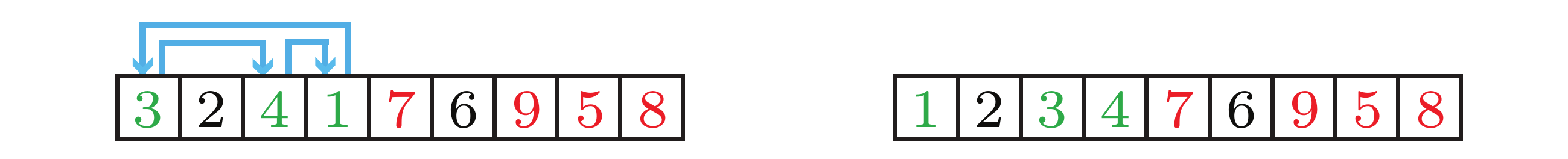}
\end{overpic}
\end{center}
\caption{\label{fig:cf-ex}[left] The initial configuration of an \plor instance with two cycles: $(341)$ and $(7985)$. A plan is shown that rearranges items following the cycle $(341)$. [right] The resulting intermediate configuration.}
\end{figure}

These cycles are uniquely determined by the initial configuration $\pi$. Noticing that each cycle requires one more pick-n-swap than the number of items in the cycle, the instance is solved using a minimum $(3 + 1) + (4 + 1) = 9$ pick-n-swaps. After processing all cycles, the \plor instance is solved and the end-effector returns to its rest position (cell $1$). The straightforward algorithm \cyclesweep, a natural greedy algorithm, is outlined in Alg.~\ref{alg:cycle-sweep}, in which the routine $\swap(\ell, i, j)$ will pick up item $j$ at cell $\ell$ (if it is not $\varepsilon$, denoting a phantom item) and swap it with item $i$ being held (if it is not $\varepsilon$). It is clear that \cyclesweep runs in linear or $O(m)$ time. 

\begin{algorithm}
\begin{small}
\vspace{0.025in}
\Comment{\footnotesize are there more cycles?}
\vspace{0.025in}
\While{there are more items to be rearranged}{
\vspace{0.025in}
\Comment{\footnotesize follow \& resolve the leftmost cycle}
$i \leftarrow $ leftmost $i$ where $\pi_i \ne i$

\swap($i$, $\varepsilon$, $\pi_i$); $g \leftarrow \pi_i$

\While{$g \ne i$}{
\swap($g$, $g$, $\pi_g$); $g \leftarrow \pi_g$
}
\swap($i$, $i$, $\varepsilon$)
}
\vspace{0.025in}
\caption{\cyclesweep($\pi$)} \label{alg:cycle-sweep}
\end{small}
\end{algorithm}

The optimality properties of \cyclesweep are established in Proposition~\ref{p:lor-p} and Proposition~\ref{p:lor-d}. 

\begin{proposition}\label{p:lor-p}%\s
\cyclesweep minimizes the number of required pick-n-swap operations for \plor. 
\end{proposition}
\begin{proof}
To solve a given cycle, it is clear that an item on the cycle must be first picked up without any other items of the same cycle already held by the end-effector (note that the end-effector may hold items from other cycles). Therefore, for each cycle, one additional pick is unavoidable. Induction 
over the cycles of $\pi$ then proves the proposition.~\qed
\end{proof}

\begin{proposition}\label{p:lor-d}%\s
After minimizing the number of pick-n-swaps, \cyclesweep computes optimal solutions for \plor in the asymptotic sense, minimizing end-effector travel in expectation, assuming that $\pi$ is a random permutation. 
\end{proposition}
\begin{proof}%[Proposition~\ref{p:lor-d}]
Given the initial configuration $\pi$, rearranging each cycle will cause the end-effector to end at where it starts following \cyclesweep, which is the leftmost location where a cycle starts. The total distance traveled by the end-effector can be factored into (i) the distance traveled to solve each cycle, and (ii) the overhead of traveling after solving a cycle to the next, including the overhead before starting the first cycle and after completing the last cycle. For (i), because each cycle must be rearranged to minimize the number of pick-n-swaps, the distance for solving each cycle is already at the minimum possible. For (ii), the end effector travels from left to right in between solving cycles. This adds no more than $2m$ distance in total. We show that $2m$ is inconsequential in comparison to the distance incurred by (i). To compute distance incurred by (i), given a random $\pi$, for a fixed $i$, the expected distance between item $i$ (located in cell $\pi^{-1}_i$, which is uniformly randomly distributed between $1$ and $m$) and cell $i$ is 

\[
\mathbb E_i = \frac{i -1 + \ldots + 1 + 0 + 1 + \ldots + m - i}{m}.
\]

Tallying over $i$ from $1$ to $m$, the expected total distance due to resolving all cycles is then
\[
\mathbb E_{\pi} = \mathbb E_1 + \ldots + \mathbb E_m =  \dfrac{1}{m}\sum_{i=1}^{m-1} (i^2 + i) \approx \frac{m^2}{3}, 
\]
which dominates $2m$. Therefore, the total end-effector travel distance produced by \cyclesweep is optimal in the asymptotic sense, in expectation.~\qed
\end{proof}

Let $H_m$ denote the $m$-th harmonic number.\footnote{$H_m = \sum_{i=1}^m \dfrac{1}{i} \approx \log m$, where $\log$ refers to natural logarithm.} We can further estimate the expected total cost according to Eq.~\eqref{eq:cost}. 

\begin{proposition}\label{p:lor-exp}%\s
For \plor with random initial configurations, the expected total rearrangement cost is 
\begin{align}\label{eq:plor-cost}
T_{\plor}(m) \approx (m + H_m-2)c_p + \frac{m^2c_t}{3}.
\end{align}
\end{proposition}
\begin{proof}
To compute the expected total cost including the cost of pick-n-swaps, we know that the number of cycles (here cycles of size $1$are included) in a random permutation $\pi$ of $[m]$ is $H_m$ \cite{flajolet2009analytic}, the $m$-th harmonic number. Given any $\pi$, the probability of any item $i$ is already at cell $i$ is $\frac{1}{m}$. Therefore, the expected number of cycles of length $1$ is $m\dfrac{1}{m} = 1$, making the expected number of cycles containing at least two items in a random permutation $H_m - 1$. 
The total number of pick-n-swaps is then $m - 1$, the number of items that must be rearranged, plus $H_m - 1$, the extra number of pick-and-swap operations for completing cycles. This yields $m + H_m - 2$.
Adding end-effector travel, the total (time) cost of rearrangement, in expectation, is then given by~Eq.~\eqref{eq:plor-cost}.~\qed
\end{proof}

\begin{remark}
In proving Proposition~\ref{p:lor-exp}, we make the observation that the two terms of the objective function Eq.~\eqref{eq:cost} are simultaneously minimized. This implies that \cyclesweep actually guarantees global optimality in the asymptotic sense for \plor. We will make this more formal in Sec.~\ref{subsec:cycle-switch}, after introducing the optimal algorithm for \plor. 
\end{remark}

\cyclesweep, as a natural greedy algorithm, minimizes the number of pick-n-swaps. It also minimizes the end-effector travel distance in the asymptotic sense. From here, any additional gain in minimizing Eq.~\eqref{eq:cost} for \plor (and later, \pltr) must come from further minimizing the end-effector travel in a non-asymptotic manner and without increasing the number of pick-n-swaps. The non-asymptotic improvement is important in practice because the number of items that are involved is generally not very large.

\subsection{Cycle Sweeping with Cycle Switching}\label{subsec:cycle-switch}
In \cyclesweep, each cycle is followed through one by one, without switching to another cycle before one is complete. If we interleave the completion of cycles, however, end-effector travel can be shortened without adding the number of pick-n-swaps. In the example illustrated in Fig.~\ref{fig:lor-ex}, the plan by \cyclesweep uses $6$ pick-n-swaps and a total end-effector distance of $14$. The alternative plan, which breaks cycles during the rearrangement process, uses also $6$ pick-n-swaps but only a total distance of $10$.

\begin{figure}[h!]
\begin{center}
\begin{overpic}[width={\iftwocolumn 1\columnwidth \else 5in \fi},tics=5]
{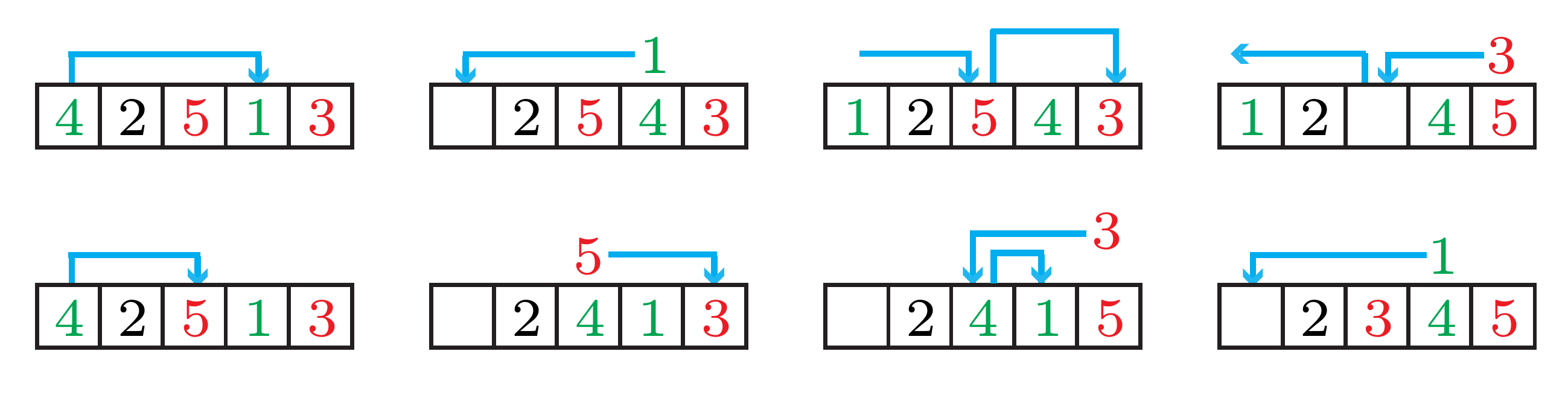}
\end{overpic}
\end{center}
\caption{\label{fig:lor-ex}[top] A rearrangement plan with a total distance of $3 + 3 + 4 + 4 = 14$ as computed by \cyclesweep. [bottom] A rearrangement plan with a total distance of $2 + 2 + 3 + 3 = 10$. Both plans require six pick-n-swap operations; the latter one does not wait for one cycle to finish.}
\end{figure}
From the example, we observe a \emph{switch} from one cycle to another cycle before completing rearranging the first can reduce end-effector travel. The saved distance corresponds to reducing the end-effector travel without holding an item. In the example, the bottom plan avoids traveling from the leftmost location to item $5$'s initial location (and back), saving a distance of $2 + 2 = 4$. 
The observation leads to the \cycleGS algorithm that groups cycles for more effective rearrangement. To describe the algorithm, some definitions are in order. Given a permutation $\pi$, let $C_{\pi}$ be the set of all cycles induced by $\pi$. For a $c \in C_{\pi}$, let $\min(c)$ and $\max(c)$ be the smallest and largest item labels of items in $c$, respectively. With a slight abuse of notation, the definitions $\min$ and $\max$ extend to a set of cycles, i.e., for $C \subset C_{\pi}$, $\min(C) = \min_{c \in C} \min(c)$ and $\max(C) = \max_{c \in C} \max(c)$. 

We group elements of $C_{\pi}$ into equivalence classes as follows. Initially, let $\mathcal C_{\pi} := \{\{c_i\} \mid c_i \in C_{\pi}\}$. Elements of $\mathcal C_{\pi}$ are grouped (via union) if their ranges overlap. That is, for two cycle groups $C_1, C_2 \in \mathcal C_{\pi}$, if their ranges $[\min(C_1), \max(C_1)]$ and $[\min(C_2), \max(C_2)]$ intersect, we update $\mathcal C_{\pi}$ to $(\mathcal C_{\pi} \backslash \{C_1, C_2\}) \cup \{C_1 \cup C_2\}$.

Since there are only a finite number of cycles, the grouping process will stop and yield a set of cycle groups that are pairwise disjoint. Then, a rearrangement plan can be computed as follows. Let the leftmost group of cycles be $C_1 = \{c_1, c_2, \ldots\}$ where $\min(c_1) < \min(c_2) < \ldots $. We start performing cycle following on $c_1$ until the end-effector passes over location $\min(c_2)$ (where $c_2$ starts) for the first time, at which point we pause following $c_1$ and switch to following $c_2$. Similarly, as $c_2$ is being followed, we will switch to following $c_3$ as the end-effector passes over $\min(c_3)$ for the first time, and so on. At some point, the end effector will reach $\max(C_1)$, the rightmost reach of the cycle group $C_1$. If there are additional cycle groups on the right of $C_1$, we pause working with $C_1$ and start working with the next cycle group on the right of $C_1$, and return to $C_1$ after all items to the right of $C_1$ are rearranged (iteratively).

An outline of the \cycleGS algorithm is given in Alg.~\ref{alg:cyclegs}, which in turn calls Alg.~\ref{alg:pcg} and Alg.~\ref{alg:pc}. It is not difficult to observe that \cycleGS runs in $O(m)$ time. We further prove its distance optimality.

\def\processCG{\textsc{ProcessCycleGroup}\xspace}
\def\processcycle{\textsc{ProcessCycle}\xspace}
\def\getcycles{\textsc{GetCycles}\xspace}

\begin{algorithm}[ht]
\begin{small}
\vspace{0.025in}
\Comment{\footnotesize retrieve cycles as singleton sets}
$\mathcal C_{\pi} \leftarrow $ \getcycles($\pi$)

\vspace{0.025in}
\Comment{\footnotesize group cycles with overlapping ranges}
\While{$\exists C_1, C_2 \in \mathcal C_{\pi}$ with overlapping ranges}{
$\mathcal C_{\pi} \leftarrow (\mathcal C_{\pi} \backslash \{C_1, C_2\}) \cup \{C_1 \cup C_2\}$ 
}

\vspace{0.025in}
\Comment{\footnotesize process cycle groups}
$C \leftarrow $ left most cycle group in $\mathcal C_{\pi}$

\processCG($\varepsilon$, $C$, $\pi$, $\mathcal C_{\pi}$)

\caption{\cycleGS($\pi$)} \label{alg:cyclegs}
\end{small}
\end{algorithm}

\begin{algorithm}[ht]
\begin{small}
\vspace{0.025in}
$c \leftarrow $ leftmost cycle in cycle group $C$

\processcycle($p$, $c$, $C$, $\pi$, $\mathcal C_{\pi}$)

\caption{\processCG($p$, $C$, $\pi$, $\mathcal C_{\pi}$)} \label{alg:pcg}
\end{small}
\end{algorithm}

\begin{algorithm}[h!]
\begin{small}
\vspace{0.025in}

$c' \leftarrow $ leftmost cycle in cycle group $C$ after $c$

$C' \leftarrow $ leftmost cycle group to the right of $C$ in $\mathcal C_{\pi}$

\vspace{0.025in}
\Comment{\footnotesize process cycle $c$ from left}
$i \leftarrow $ $\min(c)$; \swap($i$, $p$, $\pi_i$); $g \leftarrow \pi_i$ 

\While{$g \ne i$}{
\vspace{0.025in}
\Comment{\footnotesize switch cycles within the group?}
\While{$c' \ne null \textrm{ and } g > \min(c')$}{
\processcycle($g$, $c'$, $C$, $\pi$, $\mathcal C_{\pi}$)

$c' \leftarrow $ leftmost cycle in cycle group $C$ after $c'$
}
\vspace{0.025in}
\Comment{\footnotesize switch to the next cycle group?}
\If{$g == \max(C) \textrm{ and } C' \ne null$}{
\processCG($g$, $C'$, $\pi$, $\mathcal C_{\pi}$)
}

\vspace{0.025in}
\Comment{\footnotesize following the current cycle}
\swap($g$, $g$, $\pi_g$); $g \leftarrow \pi_g$
}

\swap($i$, $i$, $p$);
\caption{\processcycle($p$, $c$, $C$, $\pi$, $\mathcal C_{\pi}$)} \label{alg:pc}
\end{small}
\end{algorithm}

\begin{theorem}\label{t:cg-opt}%\s
For \plor, \cycleGS computes a rearrangement plan with minimum end-effector travel, among all plans that minimize the number of pick-n-swaps.  
\end{theorem}
\begin{proof}%[Theorem~\ref{t:cg-opt}]
We prove the theorem by showing that: (a) end-effector travel is minimal within each cycle group, and (b) end-effector travel between two adjacent cycle groups is minimized (this is trivially true). 

To prove (a), for an item with label $i$, let $\pi^{-1}_i$ be its initial location (this is natural since we then have $\pi_{\pi^{-1}_i} = i$). Notice that the minimum distance the end-effector must travel while carrying item $i$ is $|\pi^{-1}_i - i|$. In \processcycle, within a cycle group, an item $i$ is moved exactly a total distance of $|\pi^{-1}_i - i|$, even though it may be done in multiple steps. For example, in the top figure of Fig.~\ref{fig:lor-ex}, item $4$ is moved a distance of $3$ in a single move. In the bottom figure, item $4$ is first moved a distance of $2$ and later followed by a move of distance $1$.
We note that an item $i$ may be moved more than $|\pi^{-1}_i - i|$ by \cycleGS if it is carried from one cycle group to another, but such travel is attributed to travel between adjacent cycle groups, which is unavoidable.~\qed 
\end{proof}

It is clear that Proposition~\ref{p:lor-exp} and Theorem~\ref{t:cg-opt} optimize Eq.~\eqref{eq:cost} globally because each of the two terms of Eq.~\eqref{eq:cost} is optimized globally. 

\begin{corollary}\label{c:plor-global}
For \plor, in minimizing Eq.~\eqref{eq:cost}, \cyclesweep computes globally optimal solutions in the asymptotic sense in expectation and \cycleGS computes globally optimal solutions.
\end{corollary}

\begin{remark} Alg.~\ref{alg:cyclegs} has a pre-processing stage where the cycles are ordered. We note that this stage can be interleaved with the actual processing stage (Alg.~\ref{alg:pcg}). We have opted for the standalone pre-processing stage to make the algorithm hopefully more clear and the running time analysis more straightforward.
\end{remark}

\section{Partially-Labeled Rearrangement in 1D}\label{sec:por}
\def\mnmc{\textsc{OptPlanPOR}\xspace}
\def\cycleform{\textsc{FormCycles}\xspace}
\def\cyclemerge{\textsc{MergeCycles}\xspace}
\def\cyclemergeMST{\textsc{MergeCyclesMST}\xspace}
\def\cycleGSPOR{\textsc{GroupSweepCyclesPOR}\xspace}
\def\cycledist{\textsc{CycleDistance}\xspace}
\def\mst{\textsc{MST}\xspace}
\def\planptr{\textsc{PlanPTR}\xspace}

\def\cfptr{\textsc{FormCyclesPTR}\xspace}
\def\mcptr{\textsc{MergeCyclesPTR}\xspace}
\def\scptr{\textsc{SweepCyclesPTR}\xspace}

In the (fully) labeled setting, each item requiring rearranging has a single possible destination, limiting the combinatorial explosion of feasible rearrangement plans. This is no longer the case in the partially-labeled setting where each item of a given type can have multiple goal arrangements (e.g., in Fig.~\ref{fig:exs} [left], item $6$ may go to locations $4, 5$, or $6$). In other words, a partially-labeled problem can be viewed as many labeled problems mixed together, demanding additional computational efforts for selecting a best labeling to solve the problem. Nevertheless, we show that the one-dimensional partially-labeled problem \ppor can still be optimally solved despite the significant added complexity.

We describe an optimal algorithm for \ppor, \mnmc, that applies to arbitrary goal configurations. Because the algorithm is somewhat involved, for readability, we describe the algorithm over a natural but restricted class of \ppor instances where each type of item form a contiguous section in the goal configuration (see, e.g., Fig.~\ref{fig:por-phases}). 
In such instances, for a given type $1 \le t \le k$, let $n_i$ be the number of items of type $t$, $\sum_{t=1}^k n_t = m$. In the goal configuration, items of type $t$ fill locations between $\ell_t = \sum_{i=1}^{t-1} n_i + 1$ and $r_t = \sum_{i=1}^t n_i$, inclusive. Define $range(t) := [\ell_t, r_t]$. 
An instance of this restricted \ppor problem is then fully specified by the initial configuration of the items as a sequence of types, i.e., $(t_1, \ldots, t_m)$, where $1 \le t_i \le k$. 

\subsection{Algorithm Description} 
In the first phase, simple matchings are made between initial and goal locations. Starting from the left side, for the item at location $i$ with type $t_i$, $i \notin range(t_i)$, we select its goal to be the leftmost available one. Fig.~\ref{fig:por-phases} [top left] shows an example. The matchings induce a set of cycles (Fig.~\ref{fig:por-phases}, [top right], with the addition of the vertical edges corresponding to pick-n-swap operations), which are distance optimal for end-effector 
travel but do not minimize the number of pick-n-swaps. 
As is the case with \plor, each cycle requires one more pick-n-swap plus the number of items in the cycle. 
We call the subroutine \cycleform and note that additional end-effector travel between these cycles is needed for obtaining a full rearrangement plan. 

\begin{figure}[h!]
\begin{center}
\begin{overpic}[width={\iftwocolumn 1\columnwidth \else 5in \fi},tics=5]
{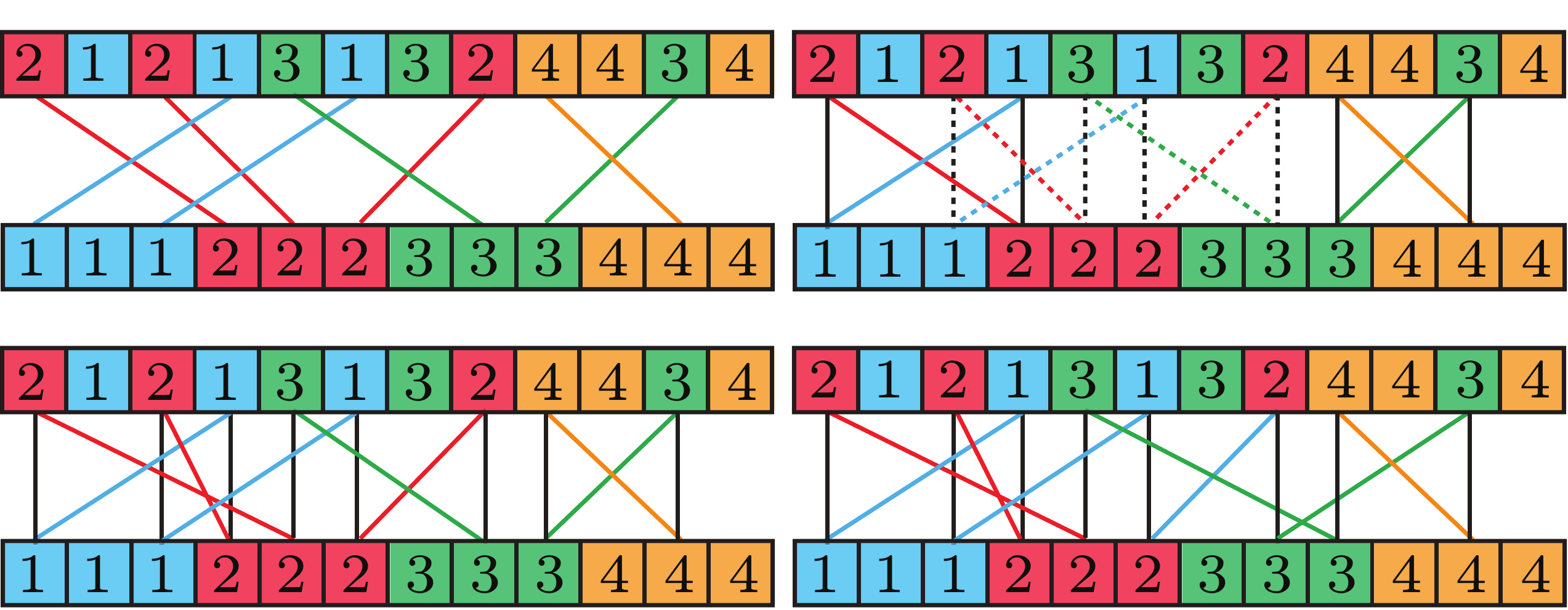}
\end{overpic}
\end{center}
\caption{\label{fig:por-phases} Illustration of the first three phases of the \mnmc algorithm. [top left] The edges show a simple matching of items with proper goals. [top right] The matchings induce three cycles, two of which are marked with solid lines and one with dashed lines. [bottom left] Swapping the two type $2$ edges on the left merges two cycles without adding end-effector travel. [bottom right] Swapping the two type $3$ edges merges two cycles but incurs additional end-effector travel costs.}
\end{figure}
In the second phase, the initial set of cycles are merged in a pairwise manner when two cycles have edges going to the same item type in the same direction (either both left or both right). Formally, two cycles can be merged in this phase if they contain two items $t_i$ and $t_j$, respectively, and $t_i, t_j$ satisfy $t_i = t_j$ and $i, j$ are either both on the left of $range(t_i)$ or both on the right of $range(t_i)$. For example, the left two cycles in Fig.~\ref{fig:por-phases} [top right] can be merged by swapping the goals of the left two type $2$ edges (first and third from the top row), yielding Fig.~\ref{fig:por-phases} [bottom left]. 
The process reduces the number of pick-n-swap operations but does not incur additional end-effector travel, because each merge keeps the total distance unchanged. We call this subroutine \cyclemerge.

In the third phase, cycles are further merged in a pairwise manner if two cycles have edges going to the same type but in different directions. Formally, two cycles can be merged in this phase if they contain two items $t_i$ and $t_j$, respectively, $t_i = t_j$, and $i, j$ are on two different sides of $range(t_i)$. 
For example, in Fig.~\ref{fig:por-phases} [bottom left], the two cycles both have edges going into the third type, but in different directions. Merging these two cycles, as shown in Fig.~\ref{fig:por-phases} [bottom right], will reduce the number of pick-n-swaps by one but will incur additional end-effector travel. We note that to ensure total distance optimality, the merge here needs to be done by computing a minimum spanning tree (MST) based on the added distances when two cycles are merged. We call the associated subroutine \cyclemergeMST. 

After the third phase completes, we obtain a set of cycles that minimizes the number of pick-n-swaps. These cycles are now much like the cycles in the labeled case and are swept through and switched similarly. We call this last subroutine \cycleGSPOR, which connects all cycles and composes the full rearrangement plan.  

\subsection{Algorithm Outline and Optimality Properties} 
The \mnmc algorithm and the \cyclemergeMST subroutine are outlined in Alg.~\ref{alg:mnmc} and Alg.~\ref{alg:cmmst}. 
The other subroutines, \cycleform, \cyclemerge, and \cycleGSPOR are relatively straightforward to implement based on the description; we omit the pseudo-code.  

\begin{algorithm}[h!]
\begin{small}
\vspace{0.025in}
\Comment{\footnotesize phase $1$: form cycles, $C$ is a list of cycles}
$C \leftarrow $ \cycleform($t_1, \ldots, t_m$)

\vspace{0.025in}
\Comment{\footnotesize phase $2$: merge cycles, w/o added distance }
$C' \leftarrow $ \cyclemerge($C$)

\vspace{0.025in}
\Comment{\footnotesize phase $3$: merge cycles, w/ added distance}
$C'' \leftarrow $ \cyclemergeMST($C'$)

\vspace{0.025in}
\Comment{\footnotesize phase $4$: group, sweep, and switch cycles}
\cycleGSPOR($C''$)

\caption{\mnmc($t_1, \ldots, t_m$)} \label{alg:mnmc}
\end{small}
\end{algorithm}

\vspace{1mm}
\begin{algorithm}[h!]
\begin{small}
\vspace{0.025in}
\Comment{\footnotesize initialize a cycle merge distance graph}
$V_C \leftarrow C$; $E_C \leftarrow \{\}$; $W \leftarrow \{\}$; $G_C \leftarrow (V_C, E_C, W)$; 

\vspace{0.025in}
\Comment{\footnotesize compute ``merge distance'' between cycles}
\For{all $c_i, c_j \in V_C$}{
$E_C \leftarrow E_C \cup \{e_{ij} = (c_i, c_j)\}$

$w_{ij} \leftarrow $ \cycledist($c_i, c_j$); $W \leftarrow W \cup \{w_{ij}\}$
}

\vspace{0.025in}
\Comment{\footnotesize compute a minimum spanning forest over $G_C$}

$F \leftarrow $ \mst($G_C$)

\vspace{0.025in}
\Comment{\footnotesize merge cycles for each tree $T$ in $F$ }

$C' \leftarrow \{\}$

\For{each tree $T$ in $F$}{
\While{$T$ has an edge $e_{ij} = (c_i, c_j)$}{
$c_i \leftarrow $ merge $c_i$ and $c_j$

collapse edge $e_{ij}$ in $T$
}

Add the single cycle $c$ in $T$ to $C'$
}

\Return $C'$
\caption{\cyclemergeMST($C$)} \label{alg:cmmst}
\end{small}
\end{algorithm}

In \cyclemergeMST, a graph $G_C$ is constructed that captures the distances between cycles that can be merged to reduce the number of pick-n-swaps. The function \cycledist computes the closest distance between two cycles for merging in the obvious way. This distance in Fig.~\ref{fig:por-phases} [bottom right] is $1$ (we note that the actual swap will incur a cost doubling this distance). After $G_C$ is constructed, which can have multiple connected components, a minimum spanning tree algorithm is executed, e.g., Prim's algorithm (\cite{prim1957shortest}), yielding a spanning forest $F$ of $G_C$. Each tree $T$ in $F$ will result in a single merged cycle; different merged cycles from different trees cannot be merged further to reduce the number of pick-n-swaps.  

We proceed to establish key properties of \mnmc and the subroutines. From here on, arbitrary goal configurations are assumed unless stated otherwise. 

\begin{proposition}\label{p:cyclemerge}%\s
For \ppor with $k$ types of items, \cyclemerge creates cycles that are distance optimal. When goals are aggregated by item types, there are at most $k - 1$ cycles after completing the \cyclemerge subroutine. 
\end{proposition}
\begin{proof}%[Proposition~\ref{p:cyclemerge}]
\cycleform creates cycles that are distance optimal by construction. 
This is clear by looking at the minimum number of times the end-effector must pass over or stop at a given cell of the lattice in order to move items of a given type to the goals. 
For each initial and goal pair, and type, this number is fully determined and realized by \cycleform. For example, for sorting the first type in Fig.~\ref{fig:por-phases}, the end-effector must pass cell $3$ at least once and must also stop at the cell at least once, because all type $1$ items should be to the left of cell $4$ in the goal configuration and there are two type $1$ items in the initial configuration to the right of cell $3$. The cycles created by \cycleform realize this minimum end-effector travel. 
Then, because \cyclemerge does not add additional distance, the total distance remains at the minimum. 

For the rest of the proposition, in the case where the goal configuration has the types aggregated, each type can participate in at most two cycles after \cyclemerge is performed. For example, in Fig.~\ref{fig:por-phases} [bottom left], two cycles exist that contain items of type $3$. There is a single cycle for types $1, 2$, and $4$.
Moreover, the leftmost and rightmost types in the goal configuration can each only participate in a single cycle. 
This is because all items of type $t$ in the initial configuration to the left (or  right) of $range(t)$ will be in a single cycle after \cyclemerge is performed, by construction.
Since each cycle requires at least two types of items, for $k$ types, there can be at most $k-1$ cycles after \cyclemerge. ~\qed 
\end{proof}

\begin{lemma}\label{l:cmmst}%\s
For \ppor, \cyclemergeMST reduces the number of pick-n-swaps to the minimum while incurring the minimum amount of additional end-effector travel. 
\end{lemma}
\begin{proof}%[Lemma~\ref{l:cmmst}]
After \cyclemerge, for each pair of the resulting cycles, they cannot be merged to reduce the number of pick-n-swaps without incurring additional end-effector travel. To see that this is the case, after \cyclemerge, for each pair of cycles, say $c_1$ and $c_2$, they can be merged to reduce the number of pick-n-swaps if and only if they contain items of the same type. Suppose that $c_1$ and $c_2$ both involve items of the same type, say $t$ (there could be multiple such types for a pair of cycles), and $c_1$ is to the left of $c_2$. 
Then it must be the case that edges of $c_1$ for restoring type $t$ items and edges of $c_2$ for restoring type $t$ items do not intersect (by the construction of \cyclemerge). The third  type in Fig.~\ref{fig:por-phases} [bottom left] gives an example. 
For a type $t$, denote the set of edges of $c_1$ (resp., $c_2$) for restoring type $t$ items as $E_1^t$ (resp., $E_2^t$). To merge $c_1$ and $c_2$, it requires for exactly one $t$, one edge of $E_1^t$ and one edge of $E_2^t$ to swap their ends so that $E_1^t$ and $E_2^t$ will cross over $range(t)$. This operation will incur additional end-effector travel that cannot be avoided. 

The optimal way to merge two cycles $c_1$ and $c_2$ sharing the same item types is by doing the merge on the type $t$ where $E_1^t$ and $E_2^t$ are closest to each other. Without loss of generality, assume $E_1^t$ is to the left of $E_2^t$. The distance between $E_1^t$ and $E_2^t$ is then simply the distance between the rightmost position reached by $E_1^t$ and the leftmost position reached by $E_2^t$. Computing this over all applicable types then yields the distance between $c_1$ and $c_2$ (this is done in \cycledist in \cyclemergeMST). 

For all cycles that can be merged into a single cycle, the merging process naturally induces a spanning tree of the involved cycles. The optimal merging sequence is then given by a minimum spanning tree as computed in \cyclemergeMST.~\qed 
\end{proof}

\begin{theorem}\label{t:por-opt}%\s
For \ppor, \mnmc computes a rearrangement plan with the minimum end-effector travel after minimizing the total number of pick-n-swaps. 
\end{theorem}
\begin{proof}%[Theorem~\ref{t:por-opt}]
By Lemma~\ref{l:cmmst}, after \cyclemergeMST, we obtain a set of cycles corresponding to the least number of pick-n-swaps and the minimum end-effector travel to realize this. What is left is to ``connect'' these cycles together to yield a complete rearrangement plan. This connection process is performed using \cycleGSPOR, which maintains the number of pick-n-swaps and adds only the minimum travel distance for cycles that are spatially disjoint. As a result, the overall \mnmc algorithm ensures distance optimality after minimizing the number of pick-n-swaps.~\qed 
\end{proof}

\begin{remark}\label{r:por}
From the discussions, theorems, and proofs, before the step of running \cyclemergeMST, \mnmc optimizes both terms of Eq.~\eqref{eq:cost} simultaneously. In executing \cyclemergeMST, in merging two cycles $c_1$ and $c_2$, let the distance cost of merging them be $dist(c_1, c_2)$. If $dist(c_1, c_2)c_t \ge c_p$, and $c_1$ and $c_2$ overlap (meaning that $c_1$ and $c_2$ can be swept without additional distance cost but with one extra pick-n-swap), then merging the two cycles will increase the total cost as defined by Eq.~\eqref{eq:cost}. In this case, we can skip the merging of $c_1$ and $c_2$. More generally, \mnmc can be augmented to globally optimize Eq.~\eqref{eq:cost} based on different $c_p:c_t$ ratios. Because the required change is fairly involved and the benefit of doing global optimization is small in comparison to optimizing Eq.~\eqref{eq:cost} sequentially (since \cyclemergeMST will only apply to a very few cycles for random instances), we omit further details. 
\end{remark}

In terms of running time, \cycleform can be performed in linear time using multiple passes over the initial and goal configurations. During the execution of \cycleform, data structures can be built so that cycles are associated with types. With the proper data structures, \cyclemerge can be run in linear time by going through the types one by one, resulting in an $O(m)$ running time. For \cyclemergeMST, computing $G_C$ and the distances between cycles can be done in linear time through amortization analysis. Computing a minimum spanning tree can be done in $O(|E_C| + |V_C|\log|V_C|)$ time \cite{johnson1975priority}. 
Merging cycles can be done at the same time as the minimum spanning tree is built, which does not take additional time. \cycleGSPOR takes $O(m)$ time. The total running time of \mnmc is then $O(m + |E_C| + |V_C|\log|V_C|)$. If the goals are aggregated based on types, there are at most $k-1$ cycles (Proposition~\ref{p:cyclemerge}) entering \cyclemergeMST, resulting in a total running time of $O(m + k \log k)$. In the general case, the running time is $O(m\log m)$.

\section{Rearrangement in 2D and Higher Dimensions}\label{sec:ltrptr}
For higher dimensions, it is straightforward to observe that the cycle-following structure for \plor and \ppor carry over. However, minimizing end-effector travel becomes more challenging, as the problem now contains a Traveling Salesperson Problem (TSP), as will be shown. Nevertheless, strategies can be derived that yield optimality in the asymptotic sense. 

\subsection{Fully-Labeled 2D (\pltr) and Higher Dimensions}
\subsubsection{ Labeled Rearrangement in 2D (\pltr).} An \pltr instance is specified by its lattice dimension $m_1, m_2$, and a permutation $\pi$ of $[m_1m_2]$. Similar to \plor, minimizing the number of pick-n-swaps for \pltr can be achieved via cycle following. This allows proving results for \pltr similar to Propositions~\ref{p:lor-p} and~\ref{p:lor-d}. A natural extension to\cyclesweep can be made to support the 2D setting: for an \pltr instance with an initial configuration $\pi$, we compute all its cycles as $c_1, \ldots, c_k$. The new algorithm, which we call \cstd, again performs cycle following of the cycles and moves to cycle $c_{i+1}$ after completing cycle $c_i$.

Then, we note that the cycle switching procedure for \plor (e.g, the process illustrated in Fig.~\ref{fig:lor-ex}) can be generalized to \pltr. That is, for consecutive items $a_1$ and $b_1$ belonging to cycle $c_1$ and an item $a_2$ belonging to cycle $c_2$, instead of going from $a_1$ to $b_1$, it can potentially save travel distance by going from $a_1$ to $a_2$, finishing $c_2$ (and possibly additional cycles), and then going to $b_1$ to finish $c_1$. The optimal switching schedule can be computed using a minimum spanning tree procedure somewhat similar to \cyclemergeMST.
There is however a key difference: in \cyclemergeMST, the cycles to be merged have symmetric distances but the distance from cycle $c_1$ to $c_2$ and the distance from cycle $c_2$ to $c_1$ are different in merging cycles for \pltr. That is, the graph where the cycles are vertices, over which a minimum spanning tree is to be constructed, is now directed. This means that we need to apply a directed minimum spanning tree algorithm~\cite{chu1965shortest,edmonds1967optimum}.
The end-effector rest position should also be considered in computing the directed minimum spanning tree. 

We call the overall cycle switching procedure for \pltr as \cswtd, which clearly runs in polynomial time. 
We omit the pseudo-code for these two algorithms given their similarity to \plor and \ppor. These algorithms are optimal in the asymptotic sense, in expectation.

\begin{proposition}\label{p:ltr-d}%\s
The number of pick-n-swaps is minimized by \cstd and \cswtd. Furthermore, they compute distance-optimal solutions in the asymptotic sense for \pltr in minimizing end-effector travel in expectation, assuming that $\pi$ is a random permutation.
\end{proposition}
\begin{proof}%[Proposition~\ref{p:ltr-d}]
It is clear that \cstd and \cswtd minimize the number of pick-n-swaps. Without loss of generality, assume $m_1 \ge m_2$. Following similar reasoning as used in the proof of Proposition~\ref{p:lor-d}, a random permutation $\pi$ will require an average distance $\mathbb E_{\pi} = \Omega((m_1+m_2)m_1m_2) = \Omega(m_1^2m_2)$. In expectation, there are $H_{m_1m_2} \approx \log m_1m_2 = O(\log m_1)$ cycles. Traveling through all these cycles once then incurs a distance cost of no more than $O(m_1 \log m_1)$, which is inconsequential as compared to $\Omega(m_1^2m_2)$.  
~\qed
\end{proof}

For $m_1 = m_2 = m$, the expected distance from cycles over a random \pltr instance can be readily computed as $((2+\sqrt{2}+5 \ln(\sqrt{2} + 1))/15) m^2 \approx 0.52 m^2$. We omit the details but note that this is equivalent to computing the average distance between two points in a unit square (see, e.g., \cite{santalo2004integral}), via a double integration, and multiply that distance by $m^2$. 
\cstd can be implemented by making a constant number of linear passes over the $m_1 \times m_2$ lattice, yielding an $O(m_1m_2)$ running time. 
\cswtd requires more work; a naive implementation requires $O(m_1^3m_2^3)$ running time, mainly for checking switching distances between cycles. 
In a sense, the optimality (in the asymptotic sense) provided by \cstd and \cswtd is the best one can do, because optimizing distance for \pltr, unlike for \plor, is NP-hard. We note that the hardness holds regardless of whether the number of pick-n-swaps is minimized. 

\begin{theorem}\label{t:ltr-hard}%\s
Minimizing the total end-effector travel distance for \pltr is NP-hard. 
\end{theorem}
\begin{proof}%[Theorem~\ref{t:ltr-hard}]
We prove the claim via a reduction from the Euclidean TSP \cite{papadimitriou1977euclidean}. Given a Euclidean TSP instance specified by a set of points embedded in a rectangular region (e.g., Fig.~\ref{fig:ltr-tsp}, left), we superimpose a lattice over the region at some resolution $m_1 \times m_2$. To construct the \pltr instance, we set the initial condition $\pi$ to be the identity permutation, i.e., $\pi_i = i$ for all $ 1 \le i \le m_1m_2$. Then, we update $\pi$ for each of the internal points (i.e., excluding the top left ``starting'' point) in the TSP instance. For a given internal point in the TSP instance, let its coordinates be $(x_1, x_2)$. Without loss of generality, we may assume that  $(x_1, x_2)$ satisfy $1 < x_1 < m_1$ and $1 < x_2 < m_2$. We update $\pi$ such that $\pi_{m_1x_2 + x_1} = m_1(x_2-1) + x_1$ and $\pi_{m_1(x_2-1) + x_1} = m_1x_2 + x_1$. That is, each internal point in the TSP instance is converted to two adjacent items (pairs of adjacent items in Fig.~\ref{fig:ltr-tsp}, right) that must be exchanged. We refer to each pair of the adjacent items to be exchanged as a \emph{cluster}.

\begin{figure}[ht]
\begin{center}
\begin{overpic}[width={\iftwocolumn 1\columnwidth \else 5in \fi},tics=5]
{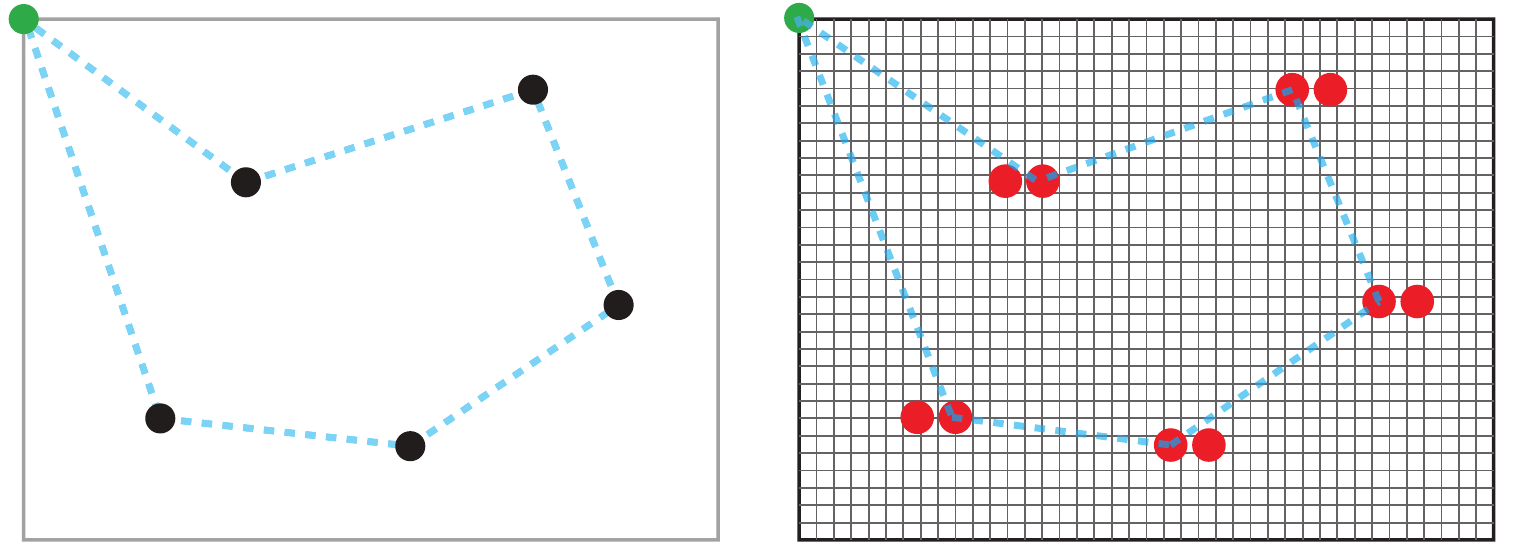}
\end{overpic}
\end{center}
\caption{\label{fig:ltr-tsp}[left] An Euclidean TSP instance, fully specified by a set of six points. [right] A corresponding \pltr instance where each point of the TSP  instance inside the rectangle is replaced by two items that must be exchanged. The top left point corresponds to the end-effector's rest position.}
\end{figure}
By selecting sufficiently large $m_1$ and $m_2$, the end-effector travel cost for solving each cluster becomes negligible. Therefore, for the \pltr instance, the optimal end-effector travel cost is determined by the cost of end-effector travel between clusters. An optimal solution to the TSP problem then maps to an optimal solution to the \pltr instance that minimizes end-effector travel (minor details are omitted). 

On the other hand, in any valid solution to the \pltr instance, the end-effector must start from the rest position (the top left point), go to each cluster to make the exchange, and eventually return to the rest position. Because the travel cost for solving each cluster locally is negligible (again, some minor but formal arguments are omitted here), a distance optimal solution then translates back to an optimal solution to the initial Euclidean TSP instance. ~\qed
\end{proof}

It is clear that the decision version of the \pltr instance is NP-complete because the distance of a given rearrangement plan can be checked in linear time.

\begin{remark}
While \plor (and \ppor) can be optimally solved in polynomial time, \pltr becomes NP-hard, implying that we cannot optimally solve it exactly in polynomial time unless P $=$ NP. What is changed is that \plor is a 1D problem whereas \pltr is a 2D problem.
Such transitions of computational complexity due to dimension changes are also observed elsewhere, e.g., computing shortest paths or the visibility graph becomes hard as dimension rises from two to three \cite{canny1987new,mitchell2004new}. 
\end{remark}

Similar to the case for \plor, we make the observation that the algorithms for \pltr simultaneously optimize the two terms in Eq.~\eqref{eq:cost}: both \cstd and \cswtd ensure that the number of pick-n-swaps is minimized and the total end-effector travel is optimal in the asymptotic sense. As a result, these algorithms provide global optimality guarantees for \pltr.  

\begin{corollary}\label{c:pltr-global}
For \pltr, in minimizing Eq.~\eqref{eq:cost}, \cstd and \cswtd compute globally optimal solutions in the asymptotic sense. 
\end{corollary}

\subsubsection{Labeled Rearrangement in Higher Dimensions.}
We briefly discuss extensions to dimensions higher than two. It is straightforward to observe that Proposition~\ref{p:ltr-d} and Theorem~\ref{t:ltr-hard} continue to hold in higher dimensions through a direct embedding, i.e., a two-dimensional problem can be readily reduced to a $d$-dimensional problem via adding additional orthogonal dimensions. 

\begin{corollary}\label{c:ltr-d}
For labeled rearrangement on $d$-dimensional lattices, with a fixed $d \ge 2$, a cycle-following procedure, after minimizing the number of pick-n-swaps, also yields a total end-effector travel distance that is optimal in the asymptotic sense, in expectation.
\end{corollary}

\begin{corollary}\label{c:ldr-hard}
Minimizing the total end-effector travel distance for labeled rearrangement on $d$-dimensional lattices, with a fixed $d \ge 2$, is NP-hard. 
\end{corollary}

\subsection{Partially-Labeled 2D (\pptr) and Higher Dimensions}
Our main focus on \pptr is finding near-optimal solutions. We present such an algorithm for minimizing the number of pick-n-swaps and show that it is distance-optimal in the asymptotic sense, in expectation, for two common goal configuration patterns shown in Fig.~\ref{fig:pptr}. In goal configuration pattern $A$, items of the same type are aggregated in both dimensions. In pattern $B$, each type occupies a single column of the lattice. 
For notational convenience, it is assumed that the number of types $k = m_1 = m_2$ is a perfect square, e.g., $k = 4 = 2^2$.
Our analysis does not depend on this last assumption.

\begin{figure}[ht]
\begin{center}
\begin{overpic}[width={\iftwocolumn 1\columnwidth \else 5in \fi},tics=5]
{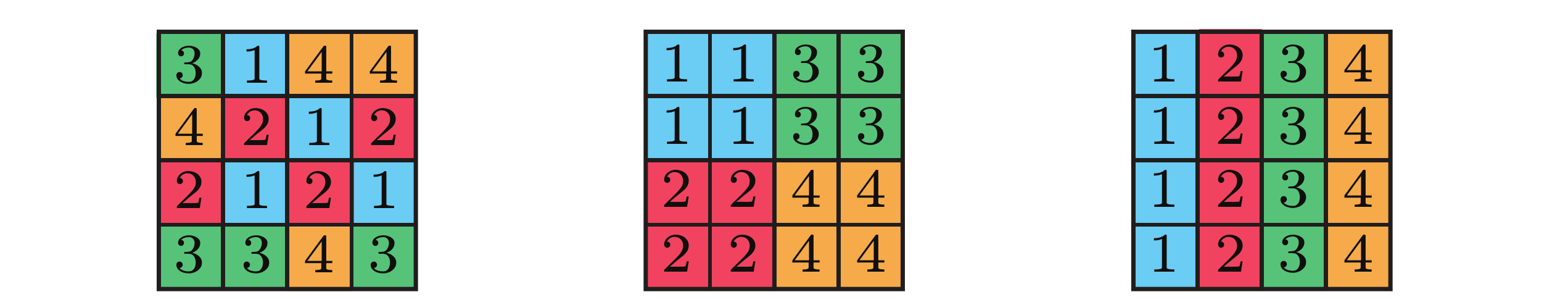}
\end{overpic}
\end{center}
\caption{\label{fig:pptr} [left] An example of a \pptr start configuration. [middle] The corresponding goal configuration pattern $A$. [right] The corresponding goal configuration pattern $B$.}
\end{figure}
For \pptr, minimizing the number of pick-n-swaps can be realized in polynomial time, as it is for \ppor: for each type, matchings can be made to yield cycles with minimum end-effector travel. These cycles can be merged using a procedure similar to \cyclemerge and \cyclemergeMST, which minimizes the number of pick-n-swaps. 

On the other hand, similar to \pltr, minimizing the end-effector travel is computationally intractable, regardless of whether the number of pick-n-swaps is minimized. This is true because  \pltr is a special case of \pptr. 

\begin{corollary}\label{c:ptr-hard}
Computing a rearrangement plan for minimizing the travel distance for \pptr is NP-hard. 
\end{corollary}

We note that we can also show that the hardness results continue to hold for goal configuration patterns $A$ and $B$. For pattern $A$, we may apply the same proof used for proving Theorem~\ref{t:ltr-hard} by treating each $\sqrt{k}\times \sqrt{k}$ cell as a single cell, which reduces the \pptr problem to an \pltr problem. For pattern $B$, the proof for Theorem~\ref{t:ltr-hard} directly applies. 

Next, we  describe a general algorithm for \pptr, which does not depend on the goal configuration pattern. The algorithm is similar to \mnmc for \ppor but with only three phases; the two cycle-merging phases in \ppor are merged into a single phase. In the first phase, for each type of item that needs to be rearranged, a bipartite graph is constructed that connects all applicable initial configurations to all goal configurations, with the edge weight being the distance between a pair of start and goal configurations. Matching is then performed on this bipartite graph to get a $1$-$1$ mapping between initial and goal configurations. This can be done efficiently using the Hungarian algorithm \cite{kuhn1955hungarian}. After the first phase, denoted as \cfptr, a set of initial cycles, $C$, is formed. 

In the second phase, cycles in $C$ are merged through goal swaps. Two cycles can be merged through goal swaps (similar to what is done in Fig.~\ref{fig:por-phases} [bottom left] $\to$ Fig.~\ref{fig:por-phases} [bottom right]) if they contain items of the same type. Unlike \ppor, where some merges do not change end-effector travel distance, a merge here will cause the total distance to increase in general. Therefore, only a single cycle merging phase is done for \pptr. The procedure for doing so is the same as \cyclemergeMST for \ppor: a graph $G$ is constructed where each cycle is a node and there is an edge between each pair of mergeable cycles with the edge weight being the additional distance that is incurred for the merge, due to goal swaps. A minimum spanning forest is then constructed for merging all mergeable cycles. After the second phase, which we call \mcptr, no more than $k/2$ cycles are left and the number of pick-n-swaps is minimized. Let this set of cycles be $C'$. 
 
In the third and last phase, cycles in $C'$ are again connected to form a complete rearrangement plan. This is realized through another round of minimum spanning tree computation, for which another graph $G'$ is needed for capturing the distance between cycles. Here, for two cycles $c_1$ and $c_2$, the distance between them is simply $\min_{v_1\in c_1, v_2 \in c_2} dist(v_1, v_2)$. Similar to the \pltr case, here, the distance between two cycles is \emph{directed}. Therefore, $G'$ is also directed. We also include $p_0$, the end-effector's initial location, as a vertex in $G'$ and compute its distance to cycles in $C'$ in the same way. After the minimum spanning tree $T$ is computed for $G'$, a double covering of this tree starting at $p_0$, going through the cycles and back, yields a complete rearrangement plan with the minimum number of pick-n-swaps. Denoting the last phase as \scptr and the overall algorithm \planptr (we omit the algorithm outline since it is fairly similar to \mnmc), we proceed to analyze the distance optimality. 

\begin{theorem}\label{t:pptr-opt}%\s
\planptr computes distance-optimal solutions in the asymptotic sense for \pptr with goal configuration patterns $A$ and $B$, in expectation.  
\end{theorem}
\begin{proof}%[Theorem~\ref{t:pptr-opt}]
We first examine pattern $A$. Intuitively, the required amount of distance for moving items to a proper goal dominates other distances. To establish this, we estimate the different costs. For a single item of a given type, the initial location can be anywhere in the lattice. Therefore, the expected distance for restoring it, regardless of where the goal is, is $\Omega(k)$. The total cost, in expectation, is then $\mathbb E = \Omega(k^3)$. It is clear that the cycles computed by \cfptr, in expectation, have a total length no more than $\mathbb E$. Because the \mcptr subroutine only swaps goals within a $\sqrt{k}\times \sqrt{k}$ square region, each swap will add at most $O(\sqrt{k})$ additional distance (we omit the straightforward computation based on the triangle inequality). Therefore, \mcptr will add at most $O(k^{5/2})$ distance. In the last phase, because at most $k/2$ cycles are connected, \scptr incurs an additional connection distance cost of $O(k^2)$. Because \mcptr and \scptr only add costs that are asymptotically inconsequential as compared to $\mathbb E$, \planptr is distance-optimal in the asymptotic sense for \pptr in expectation with goal configuration pattern $A$. 

For pattern B, it is clear that the expected cost remains at $\mathbb E = \Omega(k^3)$. For \mcptr, although goal swaps may happen over a distance of up to $k$, we note that no two different swaps will ever cross each other in the vertical direction. We readily see that (again, via an application of the triangle inequality) the cumulative distance increase per type is bounded by $2k$. The total additional cost over all types due to \mcptr is then bounded by $O(k^2)$. For pattern $B$, \scptr essentially goes from the leftmost column to the rightmost and back, which incurs $O(k)$ distance. Therefore, \planptr is distance-optimal in the asymptotic sense, in expectation, for \pptr with goal configuration pattern $B$. ~\qed
\end{proof} 

\begin{remark}\label{r:pptr}
Similar to the labeled setting, the structural results obtained for \ppor and \pptr readily extend to higher dimensions.
For patterns $A$ and $B$, it is clear that \planptr optimizes Eq.~\eqref{eq:cost} globally. 
\end{remark}

\section{Simulation Studies}\label{sec:eval}
In this section, based on simulation studies, we highlight some properties of the rearrangement problems and corroborate the guarantees provided by our algorithms. We implemented all algorithms described in the paper in Python. Each data point presented in a figure is an average of over $100$ randomly-generated instances according to some distribution to be stated. 
We mention that our basic Python implementation is fairly efficient; each instance, with some containing $10000$ items, is solved within $1$ second. For practical-sized problems with a few hundred items, each takes less than $10^{-3}$ second to solve. We do not present the computation time here as it will not be representative of an optimized implementation in C/C++. 
The source code, with implementations of both greedy and optimized/optimal algorithms for \plor/\ppor/\pltr/\pptr problems, is available at \url{https://github.com/arc-l/lattice-rearrangement/}.

For \plor and \pltr, since cycle following is a natural strategy that minimizes the number of pick-n-swaps, only  end-effector travel is examined here. 
In Fig.~\ref{fig:lor-simu} [left], \plor instances are generated following the uniform random distribution. We then take the end-effector travel distance computed by \cycleGS and divide it by $m^2$. The figure shows that the ratio converges to $1/3$ (the gray dotted horizontal line) as expected. We further note that the ratio's range, between $0.33$ and $0.345$, is very small. A contributing reason that the total travel cost is close to $m^2/3$, even when $m$ is small, is that there are not many cycles so the distance due to traveling between cycles is very minimal. 

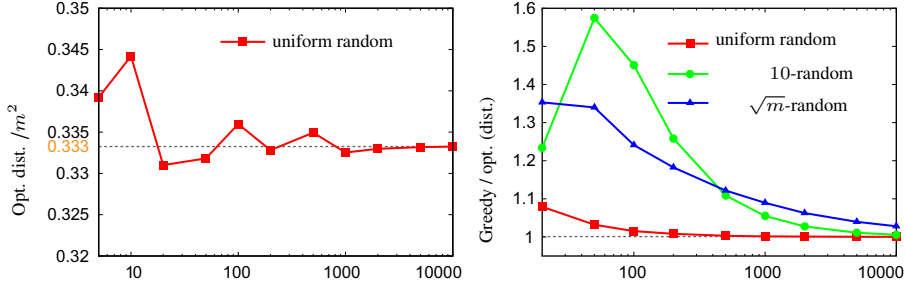
\begin{figure}[h!]
\begin{center}
\begin{overpic}[width={\iftwocolumn 1\columnwidth \else 5in \fi},tics=5]
{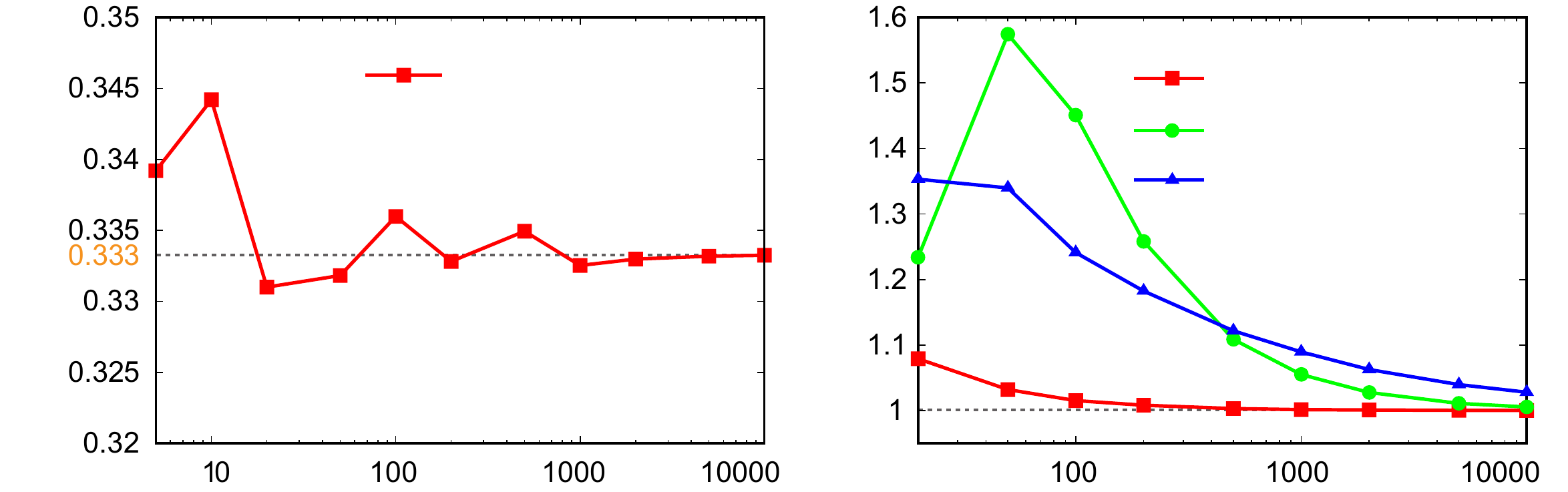}
\put(29, 26){{\scriptsize uniform random}}
\put(77.6, 26.2){{\scriptsize uniform random}}
\put(0,7.5){\rotatebox{90}{{\scriptsize Opt. dist. $/ m^2$}}}
\put(51.7,4.5){\rotatebox{90}{{\scriptsize Greedy / opt. (dist.)}}}
\put(83.5, 22.5){{\scriptsize $10$-random}}
\put(81.2, 19){{\scriptsize $\sqrt{m}$-random}}
\end{overpic}
\end{center}
\caption{\label{fig:lor-simu}Two plots illustrating properties of \plor and the associated algorithms. The $y$-axes are unitless. The $x$-axes are the number of items in an instance, which is the case for all figures in this section. [left] The optimal end-effector travel distance for \plor (computed by \cycleGS) divided by $m^2$ where $m$ is the number of items. [right] End-effector travel distance ratio between \cyclesweep and \cycleGS (optimal) for three different initial arrangement patterns.}
\end{figure}

In Fig.~\ref{fig:lor-simu} [right], the ratio of the travel distance between \cyclesweep (non-optimal) and \cycleGS (optimal) is evaluated over three item distribution patterns: uniform random, $10$-random, and $\sqrt{m}$-random, where $x$-random means that every block of $x$ items, counting from the left, are uniformly randomly distributed in the generated \plor instances. \cycleGS does significantly better than the greedy (best-first) \cyclesweep, especially when the number of items $m$ is small, which actually corresponds to more practical settings. 
For the $10$-random and $\sqrt{m}$-random settings, there are more cycles due to the partitioning, therefore allowing more opportunities for cycle switching to engage in \cycleGS, providing more distance savings as a result.

For \pltr, since cycle following is again natural, we focus on end-effector travel (all plans have optimal numbers of pick-n-swaps). On an $m \times m$ square lattice where $m$ is also a perfect square, we evaluate the performance of cycle-following algorithms over three distributions: \emph{uniform random}, \emph{column random}, where items are uniformly randomly distributed within columns, and \emph{block random}, where items are randomized within $\sqrt{m}\times\sqrt{m}$ blocks. Fig.~\ref{fig:ltr-simu} [left] presents the ratio of the distance from all cycles divided by $m^2$, which is the same for \cstd and \cswtd. That is, traveling between cycles is not included. We observe that both uniform random and block random settings have travel distances that converge to about $0.52m^2$ as predicted, whereas the column random setting, essentially a one-dimensional problem, shows convergence to $m^2/3$, also as expected.
\begin{figure}[h!]
\begin{center}
\begin{overpic}[width={\iftwocolumn 1\columnwidth \else 5in \fi},tics=5]
{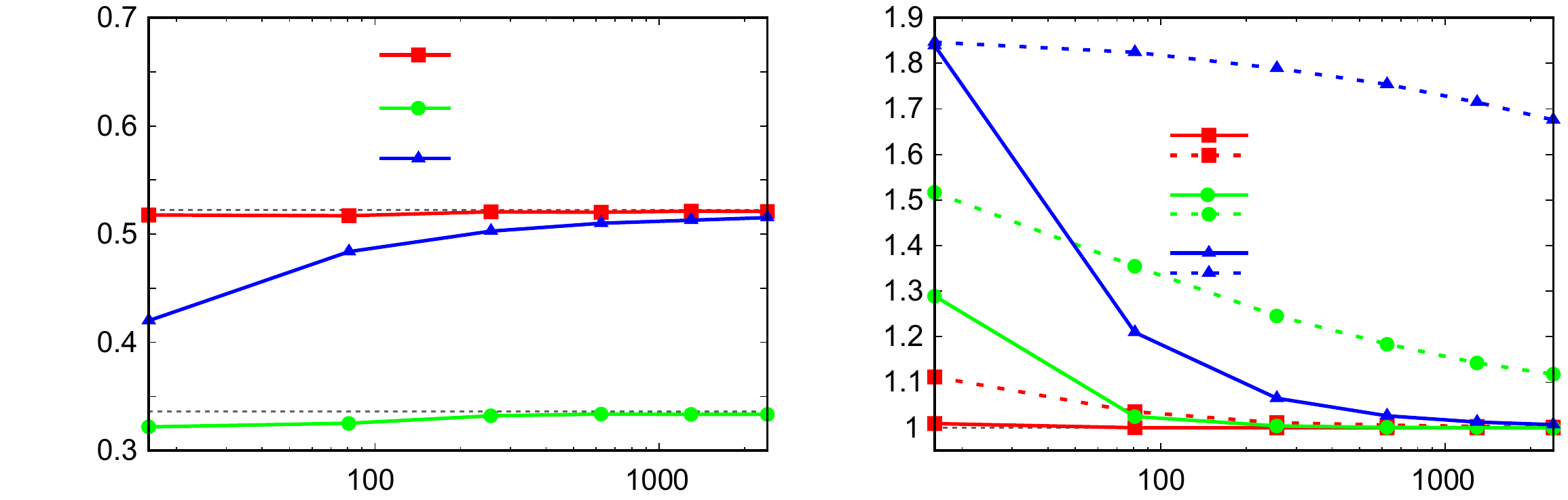}
\put(0,7.5){\rotatebox{90}{{\scriptsize Cycle dist. $/ m^2$}}}
\put(51.7,5.5){\rotatebox{90}{{\scriptsize Total / cycles (dist.)}}}
\put(29.4, 27.6){{\scriptsize uniform random}}
\put(30.1, 24.3){{\scriptsize column random}}
\put(32.3, 21.2){{\scriptsize block random}}
\put(79.8, 22.3){{\scriptsize uniform random}}
\put(80.5, 18.5){{\scriptsize column random}}
\put(82.7, 14.65){{\scriptsize block random}}
\end{overpic}
\end{center}
\caption{\label{fig:ltr-simu}Properties of \pltr and the associated algorithms. [left] Ratio of distance incurred by following cycles versus $m^2$ for three item distributions. Two gray dotted horizontal lines at around 0.52 and $1/3$ are added for reference.  [right] Total distance from algorithms versus the total cycle distances for three distributions. Data for \cstd is plotted using dashed lines and data for \cswtd is plotted using solid lines.}
\end{figure}
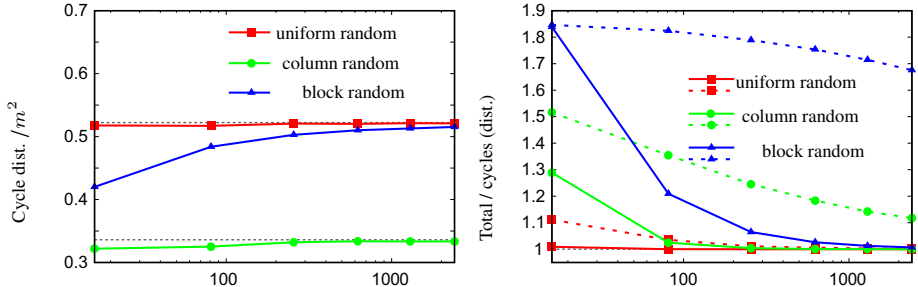

In Fig.~\ref{fig:ltr-simu} [right], for each randomness setting, the solid (resp., dashed) line shows the ratio between the distance cost from \cswtd (resp., \cstd) and the distance from cycles only. \cswtd clearly outperforms the greedy \cstd algorithm in all cases. For both uniform random and column  random, \cswtd incurs little extra distance beyond the necessary distance needed for following cycles. For the uniform random setting, different cycles are all entangled in the $m\times m$ square, allowing plenty of opportunities for cycle switching. Similarly, for the column random setting, edges of a cycle in a column are likely to pass closely by a vertex of another cycle in a nearby column, presenting opportunities for cycle switching. 
Cycles in different blocks in the block random setting tend to be more apart; it less likely for the edge of a cycle in an $\sqrt{m}\times \sqrt{m}$ block to pass closely by the vertex of another cycle in another block.

For \ppor and \pptr, we look at both the end-effector travel distance and the number of pick-n-swaps. The ratios of distance and the number of pick-n-swaps between the greedy algorithm and \mnmc are given in Fig.~\ref{fig:por-simu} for different numbers of item types. The optimal \mnmc algorithm is $1$-$2\%$ better than the greedy algorithm on distance, and up to over $5\%$ better on the number of pick-n-swaps (which carries more importance). 
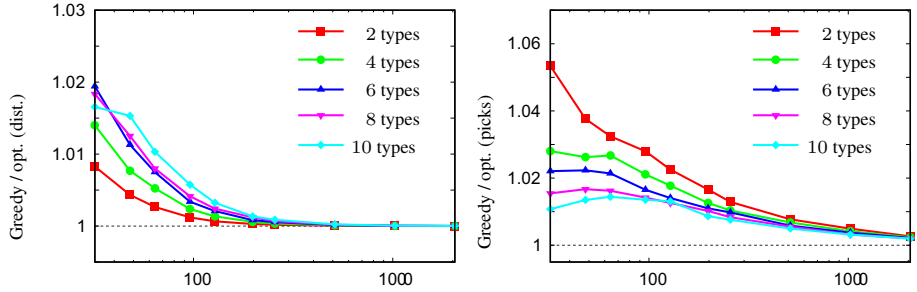
\begin{figure}[h!]
\begin{center}
\begin{overpic}[width={\iftwocolumn 1\columnwidth \else 5in \fi},tics=5]
{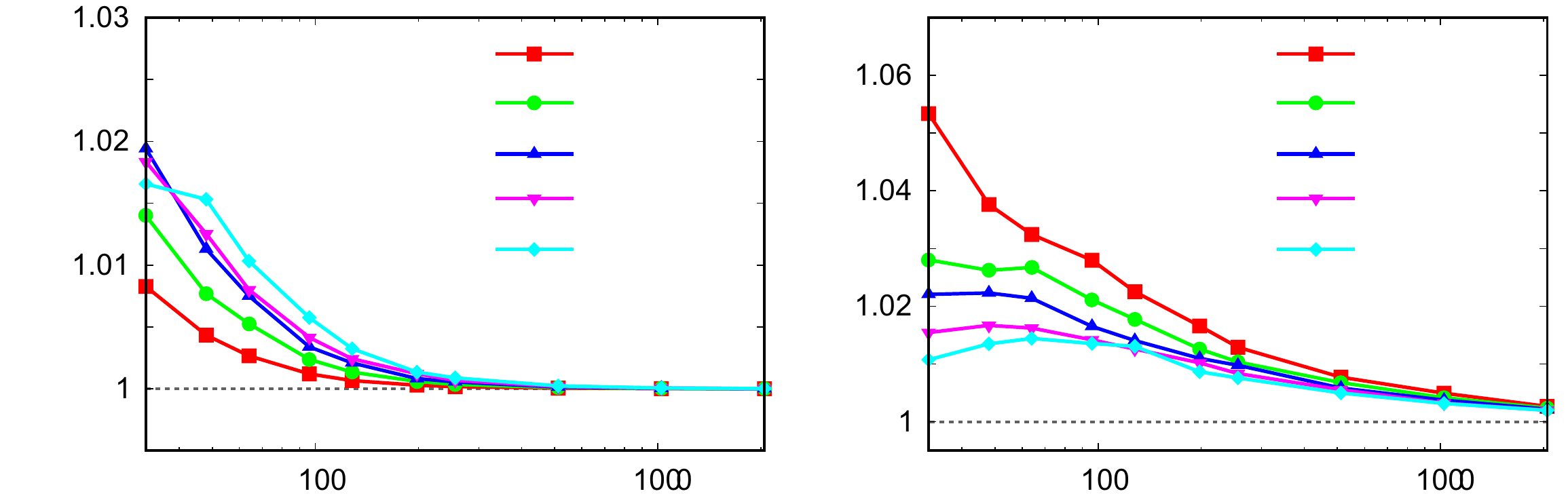}
\put(0,5.5){\rotatebox{90}{{\scriptsize Greedy / opt. (dist.)}}}
\put(51.3,4.5){\rotatebox{90}{{\scriptsize Greedy / opt. (picks)}}}
\put(39, 27.5){{\scriptsize $2$ types}}
\put(39, 24.4){{\scriptsize $4$ types}}
\put(39, 21.3){{\scriptsize $6$ types}}
\put(39, 18.2){{\scriptsize $8$ types}}
\put(37.4, 15.1){{\scriptsize $10$ types}}
\put(89, 27.5){{\scriptsize $2$ types}}
\put(89, 24.4){{\scriptsize $4$ types}}
\put(89, 21.3){{\scriptsize $6$ types}}
\put(89, 18.2){{\scriptsize $8$ types}}
\put(87.4, 15.1){{\scriptsize $10$ types}}
\end{overpic}
\end{center}
\caption{\label{fig:por-simu} Performance of the greedy algorithm versus the optimal algorithm (\mnmc) for \ppor.  [left] End-effector travel distance ratios for different numbers of types. [right] Ratios between the number of pick-n-swaps for the two algorithms for different numbers of types.}
\end{figure}

Fig.~\ref{fig:por-simu} shows that \mnmc is consistently more effective in reducing both the number of pick-n-swaps and the end-effector travel distance as compared with a greedy algorithm. This is not surprising. For \ppor, due to many items having the same types, greedy strategies tend to work more locally, leading to sub-optimality on both the number of pick-n-swaps and the end-effector travel.

For \pptr, while we no longer have algorithms for computing the optimal distance (recall that the problem is NP-hard), our minimum spanning tree-based algorithm still demonstrates a much better performance when compared to greedy best-first approaches, as shown in Fig.~\ref{fig:ptr-simu}, for both patterns (pattern $A$ and pattern $B$). The cost differences, in this case, can be similarly explained as for the \ppor case, where greedy approaches tend to work more locally.

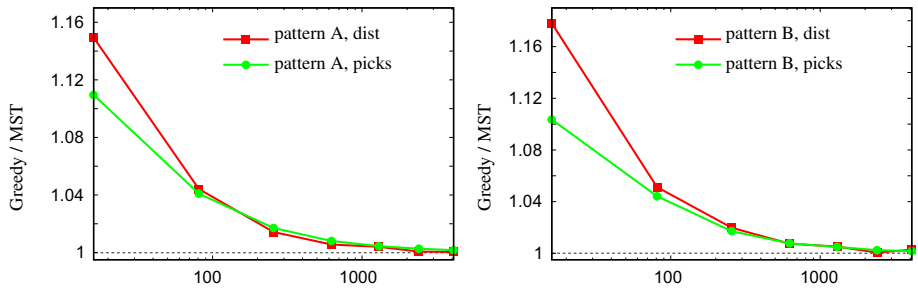
\begin{figure}[h!]
\begin{center}
\begin{overpic}[width={\iftwocolumn 1\columnwidth \else 5in \fi},tics=5]
{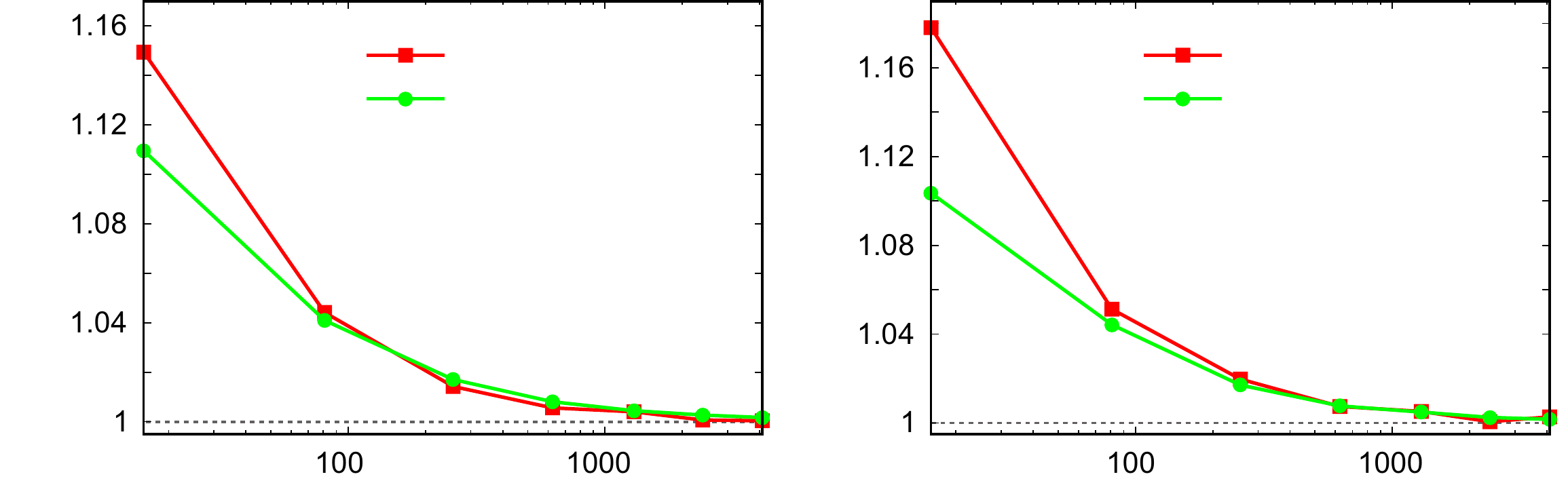}
\put(0,8.5){\rotatebox{90}{{\scriptsize Greedy / MST}}}
\put(51.3,8.5){\rotatebox{90}{{\scriptsize Greedy / MST}}}
\put(29, 27.3){{\scriptsize pattern A, dist}}
\put(29, 24.1){{\scriptsize pattern A, picks}}
\put(78.5, 27.3){{\scriptsize pattern B, dist}}
\put(78.5, 24.1){{\scriptsize pattern B, picks}}
\end{overpic}
\end{center}
\caption{\label{fig:ptr-simu} Performance of the greedy algorithm versus the minimum spanning tree-based cycle merging algorithm for \pptr.  [left] Distance and pick-n-swap ratios for ``block'' item distribution. [right] Distance and pick-n-swap ratios for `column'' item distribution.}
\end{figure}

\section{Conclusions and Discussions}\label{sec:conclusion}
In this paper, we have performed a systematic study of lattice-based robotic rearrangement using the pick-n-swap primitive. For both the fully-labeled and the partially-labeled settings under all lattice dimensions, we either provide efficient algorithms for optimally solving the problem (i.e., \plor, \ppor), or provide algorithms that are optimal in the asymptotic sense when the problem is NP-hard (\pltr, \pptr). We have demonstrated, via simulation, that our algorithms perform fairly well with respect to absolute optimality measures and in comparison with the already decent greedy best-first approaches. For the majority of the settings explored, our algorithms also provide global optimality guarantees. 

In addition to providing characterization and solutions for the specific problems, our analysis points to a general solution structure for such rearrangement problems: forming cycles naturally and then optimally connecting them, e.g., using a minimum spanning tree. Combined with proper analysis, guarantees can often be obtained. We believe this general cycle-following $+$ connecting structural insight applies to enhancing the efficiency in solving rearrangement problems beyond lattice-based settings. 

We conclude the paper with some open-ended discussions. 

\textbf{Non-random item distribution}. The current study assumes that items are randomly distributed in the lattice. If item distribution is not random, in all cases, the number of pick-n-swaps can still be readily minimized. 
For overall optimality, for \plor, the guarantee by \cyclesweep no longer holds, but \cycleGS continues to compute globally optimal solutions. For \ppor, \mnmc also continuous to ensure solution optimality guarantee as before. For \pltr and \pptr, the distance optimality guarantee (in the asymptotic sense) no longer holds. We project, however, that the associated algorithms should continue to compute high-quality solutions.

\textbf{Domain topology}. The lattices examined in this work are embedded in Euclidean spaces. This assumption may be relaxed. For example, an application may call for rearranging items that form a circle. The algorithms developed in this study can be adapted to work for such scenarios with relatively minor changes. The main update surrounds the distance computation for two lattice points, which changes as the domain's topology changes. Depending on whether the end-effector travels along the circle or in straight lines between two consecutive pick-n-swaps, the optimal rearrangement plan will change. If the end-effector travels along the circle, then the optimality guarantee for \plor and \ppor continues to hold (the algorithms will require some minor modifications). If the end-effector travel along straight lines between two points on a circle, the situation is closer to the 2D setting with similar optimality guarantees. 

\textbf{Bi-criteria optimization}. In our treatment of lattice-based rearrangement problems, there exists a fairly good level of flexibility that allows balancing between the two (sometimes competing) objectives in Eq.~\eqref{eq:cost}. For example, in \ppor, a minimum total-distance feasible solution is first computed, allowing a subsequent trade-off between reducing the number of pick-n-swaps and adding additional end-effector travel. If we enforce that the rearrangement task must be completed, then \mnmc (Alg.~\ref{alg:mnmc}) computes the full relevant Pareto frontier. 
On the other hand, our algorithms do not produce the entire Pareto optimal frontier for the two objectives if partial solutions are also considered. 

\textbf{Bounded optimality}. While not a focus of this work, if it is desirable, the algorithms in this work can be shown to provide bounded optimality guarantees, in addition to ensuring optimality in the asymptotic sense. This can be achieved by comparing the extra travel distance with the minimum required distance for realizing the rearrangement task. 

\textbf{Alternative pick-n-place primitives}. We have examined a few other natural pick-n-place primitives. It would appear that the pick-n-swap model provides a very nice balance between the complexity of the system (e.g., end-effector, workspace) design and achievable efficiency. 
For example, if the end-effector cannot swap items, e.g., it moves a picked item to a temporary location if it cannot be directly placed, it will make the system twice as inefficient; it will double the number of pick-n-place operations as picks and places are always executed with end-effector travel in between. The travel distance also doubles as a result. It could be interesting to look at a dual-arm extension of this work. Because the two arms can be at two places, additional efficiency gain should be possible. 
However, because collision avoidance must be considered in a dual-arm setting, the coordination of arms becomes non-trivial and can interfere with the rearrangement planning process. Some related studies on the dual-arm setting but not for lattice rearrangement can be found in \cite{shome2021fast}. 
It is also interesting to examine when a single end-effector can hold $k\ge 2$ items, which should also allow additional efficiency gain. In our preliminary experiments, we noticed additional travel distance savings, as a function of $k$, diminish quickly as $k$ increases.
Both dual-arm and hold-$k$ settings will likely incur more computation costs.

\section*{Acknowledgments}\label{sec:ack}
This work is supported in part by NSF awards IIS-1734419, IIS-1845888, and CCF-1934924. We sincerely thank the anonymous reviewers for bringing up many insightful suggestions and intriguing questions, which have helped improve the quality and depth of the study.

%\newpage
\bibliographystyle{IEEETran}
\bibliography{bib/bib}
\end{document}